\documentclass{article}
\usepackage{amsmath,amssymb}
\usepackage{amsthm}
\usepackage{fullpage}
\usepackage{algorithm,algorithmic}
\usepackage{color}
\usepackage[usenames,dvipsnames]{pstricks}
\usepackage{epsfig}
\usepackage{pst-grad} % For gradients
\usepackage{pst-plot} % For axes
\usepackage{eqparbox}
\usepackage{authblk}
\usepackage{enumerate}

\newcommand{\cbr}[1] {
  \{#1\}}

\renewcommand{\algorithmiccomment}[1]{\bgroup\hfill\#~#1\egroup}

\newcommand{\citet}{\cite}
\newcommand{\citep}{\cite}
\newtheorem{corollary}{Corollary}
\newtheorem{lemma}{Lemma}
\newtheorem{theorem}{Theorem}
\newtheorem{definition}{Definition}
\newtheorem{proposition}{Proposition}

\DeclareMathOperator{\ELdim}{ELdim}
\DeclareMathOperator{\Ldim}{Ldim}

\DeclareMathOperator{\DIS}{DIS}
\def\calS{\mathcal{S}}
\def\calH{\mathcal{H}}
\def\calF{\mathcal{F}}
\def\calX{\mathcal{X}}

\def\calA{\mathcal{A}}
\def\calS{\mathcal{S}}
\def\calC{\mathcal{C}}
\def\calN{\mathcal{N}}
\def\R{\mathbb{R}}

\def\N{\mathbb{N}}
\def\x{\mathbf{x}}
\def\v{\mathbf{v}}
\newcommand{\IIF}[1]{\STATE\algorithmicif\ #1\ \algorithmicthen}
\newcommand{\ENDIIF}{\unskip\ \algorithmicend\ \algorithmicif}

\title{The Extended Littlestone's Dimension for Learning\\ with Mistakes and Abstentions}
\author[1]{Chicheng Zhang\thanks{chz038@eng.ucsd.edu}}
\author[1]{Kamalika Chaudhuri\thanks{kamalika@cs.ucsd.edu}}
\affil[1]{University of California, San Diego}

\begin{document}

\maketitle

\begin{abstract}
This paper studies classification with an abstention option in the online setting. In this setting, examples arrive sequentially, the learner is given a hypothesis class $\calH$, and the goal of the learner is to either predict a label on each example or abstain, while ensuring that it does not make more than a pre-specified number of mistakes when it does predict a label.

Previous work on this problem has left open two main challenges. First, not much is known about the optimality of algorithms, and in particular, about what an optimal algorithmic strategy is for any individual hypothesis class. Second, while the realizable case has been studied, the more realistic non-realizable scenario is not well-understood. In this paper, we address both challenges. First, we provide a novel measure, called the Extended Littlestone's Dimension, which captures the number of abstentions needed to ensure a certain number of mistakes. Second, we explore the non-realizable case, and provide upper and lower bounds on the number of abstentions required by an algorithm to guarantee a specified number of mistakes.

\end{abstract}

\section{Introduction}

Many machine learning applications, such as fraud detection in credit card transactions and medical diagnosis, involve high misclassification penalties. In these cases, it is often desirable to design a predictor which is guaranteed to make no more than a certain number of prediction mistakes, at the expense of a few abstentions; the examples on which the learner abstains are then passed on to a human, who can then take a closer look. 

In this paper, we study this problem -- classification with an option to abstain -- in an online setting. Examples arrive sequentially, the learner is given a hypothesis class $\calH$, and the goal of the learner is to either predict a label on each example or abstain, while ensuring that it does not make more than a pre-specified number of mistakes. The first formal theoretical framework to address this problem was the Knows What It Knows (KWIK) model~\citep{LLWS11}; assuming that labels are generated by a hypothesis in $\calH$ (the so-called {\em{realizability assumption}}), this framework requires the learner to {\em{always}} predict correctly. ~\citet{LLWS11} also provided an algorithm in this model that makes no mistakes and requires $|\calH| - 1$ abstentions for finite $\calH$. \citet{SZB10} further extended this model to allow learners that do not always have to be correct, but are permitted to make upto a bounded number of mistakes. They showed that when $k$ mistakes are allowed, the number of abstentions can be reduced to $(k+1)|\calH|^{1/(k+1)}$.

While previous work has looked at designing generic online learning algorithm in this setting, there are two remaining challenges. First not much is known about the optimality of these algorithms, and in particular, about what an optimal algorithmic strategy would be for any individual hypothesis class. The second challenge is to understand what happens in a more realistic scenario where the realizability assumption does not hold. While this has been studied in a regression setting~\citep{SS11}, not much is known about the classification case.

In this paper, we address both challenges.  We first provide a new measure that, given a hypothesis class $\calH$, captures how many abstentions are needed to ensure a certain number of mistakes, and we provide an optimal algorithm that achieves this number. Our measure is closely related to the notion of Littlestone's dimension for online learning with no abstentions, and we call it the Extended Littlestone's dimension. Formalizing this notion additionally allows us to extend our algorithm to infinite hypothesis classes; while algorithms were previously known for some specific infinite classes~\citep{SZB10}, no generic algorithm was known.

Next, we focus our attention on the non-realizable case. In this case, we make an $l$-bias assumption, which ensures that the labels are generated by a function that disagrees with some (unknown) hypothesis $h \in \calH$ on at most $l$ examples. We show that (at least some form of) this assumption is necessary; there exists a finite hypothesis class $\calH$, such that when the $l$-bias assumption holds, any algorithm that abstains a finite number of times must make at least $l$ mistakes.  Moreover, there also exists an infinite hypothesis class $\calH$ with Littlestone's dimension $d$ such that any algorithm that abstains a finite number of times must make at least $l + d$ mistakes. To complement these lower bounds, we show that we can run a version of our algorithm when the $l$-bias assumption holds, and provide an upper bound on the number of abstentions it makes.

\section{The Setting}
\label{sec:setting}

We consider the problem of online classification in the model of~\cite{SZB10}, where the learner is allowed to occasionally abstain from prediction. The precise setting is as follows. At time $t$, the adversary presents an example $x_t$ in some instance space $\calX$. The learner makes its prediction $\hat{y}_t$, which can be either $-1$, $+1$, or $\bot$ (I don't know). The adversary then reveals an outcome $y_t \in \cbr{-1,+1}$. The interaction between the learner and the adversary continues, and the performance of the learner is measured by the total number of mistakes and abstentions made throughout the process.

To help make decisions, the learner has access to a hypothesis class $\calH$. Each hypothesis $h$ in $\calH$ is a prediction rule mapping from $\calX$ to $\cbr{-1,+1}$.

\paragraph{Basic Notations.} Given two hypothesis $h_1$ and $h_2$, their {\em{product}} is defined as a new hypothesis $h_1 \cdot h_2$ which is a function that takes $x$ as input, and outputs $h_1(x) \cdot h_2(x)$.  
Given two hypothesis classes $\calH_1$ and $\calH_2$, we define $\calH_1 \cdot \calH_2$ to be the class of functions achievable by taking the product between a function in $\calH_1$ and a function in $\calH_2$. Formally,
\[ \calH_1 \cdot \calH_2 = \cbr{h_1 \cdot h_2: h_1 \in \calH_1, h_2 \in \calH_2 } \]
%Define function $I^{\pm}(A) = +1$ if $A$ is true, $I^{\pm}(A) =-1$ if $A$ is false. Also, 
Define function $I(A) = 1$ if $A$ is true, $I(A) = 0$ if $A$ is false.\\
Denote by $\calC^l$ the class of union of at most $l$ singletons in instance domain $\calX$. That is, hypotheses that take value $+1$ on $\calX$, except for at most $l$ points:
\[ \calC^l = \cbr{1-2I(x = x_1 \vee \ldots \vee x = x_i): x_1, \ldots, x_i \in \calX, i \leq l}\]
In this paper, we will address both realizable and nonrealizable cases, defined below. 

\paragraph{Realizable Case.} In the realizable case, we assume there is a hypothesis $h$ in $\calH$ that makes no mistakes over time. Formally, a sequence $S = (x_1, y_1)$, \ldots, $(x_n, y_n)$ is called $\calH$-realizable, if and only if
\[ \exists h \in \calH, \quad |\cbr{t: h(x_t) \neq y_t}| = 0 \]
%\[ \calS = \cbr{(x_1,y_1), \ldots, (x_n, y_n) \in (\calX \times \cbr{-1,+1})^* : \exists h \in \calH, \quad |\cbr{t: h(x_t) \neq y_t}| = 0 } \]
%\kc{You may want to say instead:
%\[ \exists h \in \calH, \quad |\cbr{t: h(x_t) \neq y_t}| = 0 \]
%Similarly for the non-realizable case.}

%\kc{We are using a slighly different version of this assumption, right? $l$-bias assumption?} 
%\cz{The following is not the most general assumption. More generally we would like the sequence $S = (x_1, y_1)$, \ldots, $(x_t, y_t), \ldots$ satisfying for all $x$, $|\cbr{y | \exists t, (x,y_t) \in S}| \leq 1$.}
\paragraph{Non-realizable Case.} In the non-realizable case, we assume that the label of examples are generated by a function that disagrees with some hypothesis on at most $l$ examples, which we call $l$-bias assumption.\footnote{This is similar to the $l$-mistake assumption in the expert problem~\citep{CBFHW96,ALW06}.} Formally, define $\calH^l$ as the set of classifiers where its prediction differs from some classifier in $\calH$ on at most $l$ points, i.e. $\calH^l = \calH \cdot \calC^l$.
A sequence $S = (x_1, y_1)$, \ldots, $(x_n, y_n)$ is said to have $l$-bias with respect to $\calH$, if and only if it is $\calH^l$-realizable, i.e.
%\cz{Define $\calH^l$ here, and say: there exists some $h$ in $\calH^l$ that classifies all examples correctly?}
%\[ \forall x \in \cbr{x_1, \ldots, x_n}, |\cbr{y: (x,y) \in S}| \leq 1\]
\[ \exists h \in \calH^l, \quad |\cbr{t: h(x_t) \neq y_t}| = 0 \]
%, the adversary is only allowed to show sequences in $\calS$, where:
%\[ \calS = \cbr{(x_1,y_1), \ldots, (x_n, y_n) \in (\calX \times \cbr{-1,+1})^* : \exists h \in \calH, \quad |\cbr{t: h(x_t) \neq y_t}| \leq l } \]

\paragraph{Version Space and Disagreement Region.} In the realizable case, it is often convenient to consider the set of hypotheses that agree with the labeled examples revealed so far. Given an labeled set $S$ and a set of hypotheses $V$, $V[S] \subseteq V$ is defined as the set of all classifiers that classify $S$ correctly:
\[ V[S] = \cbr{h \in V: \text{for all } (x,y) \in S, h(x) = y}\]
At the start of time $t$, the version space is defined as the set of hypotheses in $\calH$ that agree with the examples $(x_1,y_1)$, \ldots, $(x_{t-1},y_{t-1})$ seen so far. 

We say that an example $x$ is in the disagreement region of a hypothesis set $V$, denoted by $\DIS(V)$, if both $V[(x,+1)]$ and $V[(x,-1)]$ are nonempty.

%\kc{Need a sentence to more concretely express why it is ok to consider only these algorithms.}
%\cz{The reason is that, this will only reduce the version space without triggering new mistakes and abstentions. Equivalently, we focus on algorithms that do not update their internal states when they make a correct prediction. }
\paragraph{Nontrivial Rounds.} For deterministic learners, it is always suboptimal for the adversary to present an example on which the learner predicts correctly, since this will only impose additional constraints on examples shown in the future without changing the number of mistakes and abstentions made. A round $t$ is called {\em nontrivial} if and only if the learner makes a mistake or abstains on that round.
%Hence without loss of generality, we assume all rounds $t$ are nontrivial.

%Following~\cite{CBFHW96}, we adopt the {\em conservative principle} in the design of online prediction algorithms.
%When the algorithm predicts a label correctly, that is, when $\hat{y}_t = y_t$, we skip this round and pretend it does not happen. We call the rest of the rounds {\em nontrivial}, as it triggers updates in the internal state of the algorithm. Hence without loss of generality, we assume for all rounds $t$, the algorithm either makes a mistake, or predicts $\bot$. 
%\cz{For simplicity in this paper we only consider deterministic prediction; we leave randomized prediction in future work.} \kc{This paragraph could be shortened by editing.}

\paragraph{The Mistake Bound Model.} The mistake bound model~\citep{BF72,L87,A87} is a central model in online classification. 
An online learning algorithm $\calA$ achieves a mistake bound $M$ with respect to a set of sequences $\calS \subseteq (\calX \times \cbr{-1,+1})^*$ if and only if for any adversary showing sequences $S = ((x_1, y_1),\ldots,(x_n, y_n))$ in $\calS$, $\calA$'s prediction $\hat{y}_1$, $\ldots$, $\hat{y}_n$ $\in \cbr{-1,+1}$ satisfies
\[ \sum_{t=1}^n I(\hat{y}_t \neq y_t) \leq M\] 

\paragraph{The $k$-SZB Model.} In this paper, we consider $k$-SZB model studied by~\cite{SZB10, DZ13}, which extends the mistake bound model by additionally allowing the learner to say ``Don't Know''($\bot$).
%provides a tradeoff between the mistake bound model~\cite{L87} and the KWIK model~\cite{LLWS11}. Note that the learner now has the additional option to say ``Don't Know''($\bot$). 
An online learning algorithm $\calA$ achieves a $(k,m)$-SZB bound with respect to a set of sequences $\calS \subseteq (\calX \times \cbr{-1,+1})^*$, if and only if for any adversary that presents sequences $S = ((x_1, y_1),\ldots,(x_n, y_n))$ in $\calS$, $\calA$'s prediction $\hat{y}_1$, $\ldots$, $\hat{y}_n$ $\in \cbr{-1,+1,\bot}$ satisfies
\[ \sum_{t=1}^n I(\hat{y}_t =- y_t) \leq k\] 
\[ \sum_{t=1}^n (I(\hat{y}_t =- y_t) + I(\hat{y}_t = \bot)) \leq m\]
In other words, the number of mistakes is at most $k$, and number of nontrivial rounds (where the algorithm makes a mistake or abstains) is at most $d$. When $k = 0$, this is exactly the KWIK model~\citep{LLWS11}. We will look at this model in both the realizable and non-realizable cases.
%\cz{Alternatively we can bound the number of nontrivial rounds as opposed to the number of $\bot$'s.}
%\kc{This model should have been described precisely in the Settings section. Perhaps have a small part of this section where you describe (a) the MB model and (b) the SZB model after you talk about realizable and non-realizable case. Then say that we will look at this model in both the realizable and non-realizable case.}

\section{Extended Littlestone's Dimension}
\label{sec:realizable}

%\kc{Begin this section by saying that for the rest of the section, you will assume that you are in the realizable case.}
%Throughout the rest of the section, we assume that we are in the realizable case.
%We begin with our results for the realizable case.
%First we define extended mistake tree, which is a natural generalization of mistake tree in mistake bound model.
%We use the extended mistake tree to characterize the optimal number of nontrivial bounds in $k$-SZB model given hypothesis class $\calH$, which is called $\calH$'s extended Littlestone's dimension.
%Additionally, we present an optimal algorithm for online prediction in $k$-SZB model, namely SOA.DK (Algorithm~\ref{alg:esoa}).

We begin with the realizable case and the definition of the Extended Littlestone's Dimension. We first define an extended mistake tree, which is a natural generalization of the mistake tree, and then use it characterize the optimal number of non-trivial rounds (abstentions + mistakes) for any algorithm in the $k$-SZB model. We finally present an optimal algorithm (Algorithm~\ref{alg:esoa}) for this model, and a recursive formulation of Extended Littlestone's dimension.

\subsection{Background: Mistake Bound, Littlestone's Dimension and Standard Optimal Algorithm}
\label{sub:mb}

%\kc{This subsection is fine otherwise, except for the fact that it is a little unclear how the different things you are describing relate to one another. You need some extra "connecting" text. We can discuss Tue.}

%\kc{Why do we need to state the minimax formulation? To discuss on Tue.}
%\paragraph{Minimax Formulation.} The optimal mistake bound can be rephrased as the number of mistake made if both the learner and the adversary plays optimally.
%\[ \max_{x_1} \min_{\hat{y}_1} \max_{y_1: \calH[(x_1,y_1)] \neq \emptyset} \max_{x_2} \min_{\hat{y}_2} \max_{y_2: \calH[(x_1,y_1),(x_2,y_2)] \neq \emptyset} \ldots \max_{x_n} \min_{\hat{y}_n} \max_{y_n: \calH[(x_1,y_1),\ldots,(x_n,y_n)] \neq \emptyset} \sum_{t=1}^n I(\hat{y}_t \neq y_t) \]

\cite{L87} provides a characterization of the optimal mistake bound in the realizable case, which is measured by Littlestone's dimension. We begin by describing this characterization.

\paragraph{Mistake Trees.} Littlestone's dimension is closely related to the notion of a mistake tree. A mistake tree~\footnote{In~\cite{L87} this is instead called a ``complete mistake tree''; in~\cite{SS12} this is called a $\calH$-shattered tree.} of a hypothesis class $\calH$ is a complete binary tree~\footnote{A complete binary tree is one in which every level is completely filled with nodes.} , whose leaves are classifiers in $\calH$ and whose internal nodes correspond to examples in $\calX$. A mistake tree may have no internal nodes, in which case it only contains a leaf corresponding to a classifier $h$ in $\calH$ -- we call it a zeroth order mistake tree. Given an internal node, the edge connecting it and its left (resp. right) child is labeled $-1$ (resp. $+1$).

A root to leaf path $p$ in mistake tree $T$ is a sequence of nodes and edges $v_1e_1v_2e_2\ldots v_ne_nv_{n+1}$, where $v_1, \ldots, v_n$ are internal nodes in $T$ corresponding to examples in $\calX$, $v_1$ is the root node of $T$, each $e_i$ is an edge in $T$ that connects $v_i$ and $v_{i+1}$, $v_{n+1} = h$ is a classifier in $\calH$ corresponding to a leaf in $T$. For each $i$, edge $e_i$ connects $v_i$ and $v_{i+1}$. The length of a path $l(p)$ is defined as the number of edges in $p$. For each leaf, the associated classifier agrees with the internal nodes and edges along the path up to the root. That is, if each node $v_i$ corresponds to example $x_i$ and each edge $e_i$ has label $y_i$, then $h$ agrees with examples $\cbr{(x_1,y_1), \ldots, (x_n, y_n)}$. See Figure~\ref{fig:mistaketree} for an illustration.
%\cz{A depth-$0$ mistake tree is one that has no internal nodes -- it has only one leaf corresponding to a classifier $h$ in $\calH$.}

\begin{figure}
\centering
% Generated with LaTeXDraw 2.0.8
% Sun Jan 10 21:41:14 PST 2016
% \usepackage[usenames,dvipsnames]{pstricks}
% \usepackage{epsfig}
% \usepackage{pst-grad} % For gradients
% \usepackage{pst-plot} % For axes
\scalebox{0.8} % Change this value to rescale the drawing.
{
\begin{pspicture}(0,-1.8792187)(11.99,1.8792187)
\usefont{T1}{ppl}{m}{n}
\rput(2.7545311,1.6757812){$3$}
\psline[linewidth=0.04cm](2.57,1.3457812)(1.63,0.34578124)
\psline[linewidth=0.04cm](2.97,1.3657813)(3.93,0.36578125)
\usefont{T1}{ppl}{m}{n}
\rput(1.8445313,0.9557812){$-1$}
\usefont{T1}{ppl}{m}{n}
\rput(3.7445312,0.97578126){$+1$}
\usefont{T1}{ppl}{m}{n}
\rput(1.4545312,0.03578125){$2$}
\usefont{T1}{ppl}{m}{n}
\rput(4.074531,0.09578125){$4$}
\psline[linewidth=0.04cm](7.51,0.22578125)(11.97,0.22578125)
\psline[linewidth=0.04cm](7.99,0.36578125)(7.99,0.12578125)
\psline[linewidth=0.04cm](9.13,0.38578126)(9.13,0.14578125)
\psline[linewidth=0.04cm](10.31,0.38578126)(10.31,0.14578125)
\psline[linewidth=0.04cm](11.47,0.38578126)(11.47,0.14578125)
\usefont{T1}{ppl}{m}{n}
\rput(7.954531,-0.16421875){$1$}
\usefont{T1}{ppl}{m}{n}
\rput(9.134531,-0.14421874){$2$}
\usefont{T1}{ppl}{m}{n}
\rput(10.294531,-0.12421875){$3$}
\usefont{T1}{ppl}{m}{n}
\rput(11.454532,-0.10421875){$4$}
\usefont{T1}{ppl}{m}{n}
\rput(0.47453126,-1.6442188){$h_1$}
\usefont{T1}{ppl}{m}{n}
\rput(2.1545312,-1.6242187){$h_2$}
\usefont{T1}{ppl}{m}{n}
\rput(3.3945312,-1.6042187){$h_3$}
\usefont{T1}{ppl}{m}{n}
\rput(5.1545315,-1.5442188){$h_4$}
\psline[linewidth=0.04cm](1.17,-0.31421876)(0.23,-1.3142188)
\psline[linewidth=0.04cm](1.57,-0.29421875)(2.53,-1.2942188)
\usefont{T1}{ppl}{m}{n}
\rput(0.44453126,-0.70421875){$-1$}
\usefont{T1}{ppl}{m}{n}
\rput(2.3445313,-0.68421876){$+1$}
\psline[linewidth=0.04cm](4.15,-0.27421874)(3.21,-1.2742188)
\psline[linewidth=0.04cm](4.55,-0.25421876)(5.51,-1.2542187)
\usefont{T1}{ppl}{m}{n}
\rput(3.4245312,-0.6642187){$-1$}
\usefont{T1}{ppl}{m}{n}
\rput(5.324531,-0.64421874){$+1$}
\end{pspicture} 
}
\caption{A mistake tree with respect to set of threshold classifiers $\calH = \cbr{h_i = 2 I(x \leq i) - 1: i =1,2,3,4}$.}
\label{fig:mistaketree}
\end{figure}
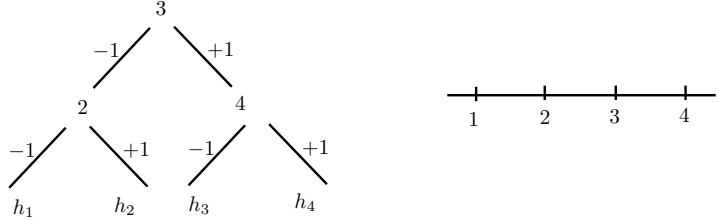

%A mistake tree $T$ succinctly represents a strategy of the adversary.
%Initially $t=1$, the adversary start with the root node. It chooses the example corresponding to the root as $x_1$, and shows it to the learner.
%If the learner predicts $\hat{y}_1 = -1$, the adversary reveals label $y_1 = +1$, follows the downward edge labeled $+1$, reaching its right child, and vice versa.
%At time $t \geq 2$, suppose the adversary has reached a node containing example $x_t$. It chooses $x_t$, and shows it to the learner.
%Same as before, If the learner predicts $\hat{y}_t = -1$, the adversary reveals label $y_t = +1$, and follows the downward edge labeled $+1$, reaching its right child, and vice versa.
%The interaction comes to an end if a leaf is reached.
%It can be seen that if the mistake tree constructed has depth $d$, then the adversary can force the learner to make $d$ mistakes using its associated strategy.

A mistake tree $T$ succinctly represents a strategy of the adversary in response to a deterministic learner.
At $t=1$, the adversary picks the example $x_1$ corresponding to the root node to show to the learner.  If the learner predicts $\hat{y}_1 = -1$, the adversary reveals label $y_1 = +1$, and follows the downward edge labeled $+1$; otherwise it follows the other edge. If at time $t \geq 2$, the adversary reaches a node with example $x_t$, then $x_t$ is shown to the learner, and one of the downward edges adjacent to this node is followed.
%At time $t \geq 2$, suppose the adversary has reached a node containing example $x_t$. It chooses $x_t$, and shows it to the learner.
%Same as before, If the learner predicts $\hat{y}_t = -1$, the adversary reveals label $y_t = +1$, and follows the downward edge labeled $+1$, reaching its right child, and vice versa.
The interaction comes to an end when a leaf is reached.
It can be seen that the adversary forces the learner to make a mistake at each node of the mistake tree; this implies that if every root-to-leaf path of the mistake tree has depth $d$, then the adversary can force the learner to make $d$ mistakes using the associated strategy.

We are now ready to define Littlestone's dimension.

\begin{definition}
The Littlestone's dimension of hypothesis class $\calH$, $\Ldim(\calH)$, is the maximum depth of any mistake tree of $\calH$.
\end{definition}
\begin{theorem}[\citep{L87}]
For a hypothesis class $\calH$, the optimal mistake bound of any deterministic algorithm with respect to adversaries showing $\calH$-realizable sequences is equal to $\Ldim(\calH)$.
\end{theorem}

\paragraph{Standard Optimal Algorithm.} Algorithm~\ref{alg:soa} presents the Standard Optimal Algorithm, which is an optimal deterministic algorithm for online classification in the realizable case. It maintains a version space $V$ over time. At each time $t$, it predicts a label $y_t$ such that each mistake will force the version space's Littlestone's dimension to drop by at least 1. Therefore, the number of mistakes made by Algorithm~\ref{alg:soa} is at most $\Ldim(\calH)$.

\begin{algorithm}[H]
\caption{Standard Optimal Algorithm~\citep{L87}}
\begin{algorithmic}[1]
\STATE Input: hypothesis class $\calH$.
\STATE Initialize version space $V \gets \calH$.
\FOR{$t = 1,2,\ldots,$}
\STATE Receive example $x_t \in \calX$.
\STATE Make prediction $\hat{y}_t = \arg\max_y \Ldim(V[(x,y)])$.
\STATE Receive label $y_t$.
\IF {$\hat{y}_t = -y_t$}
    \STATE Update version space $V \gets V[(x_t, y_t)]$.
\ENDIF
\ENDFOR
\end{algorithmic}
\label{alg:soa}
\end{algorithm}

\paragraph{Recursive Definition.} For finite $\calH$\footnote{For infinite $\calH$, the recurrence may not reach the base case.} one also has the following recurrence for its Littlestone's dimension:
\[ \Ldim(\calH) = \begin{cases} 0 & |\calH| = 1\\ 1 + \max_{x \in \DIS(\calH)} \min_{y \in \cbr{\pm 1}} \Ldim(\calH[(x,y)]) & |\calH| > 1 \end{cases}\]
%\cz{This definition is tricky, since some recursion may not even reach the base case! But it is an intuitive characterization.}

\subsection{Extended Littlestone's Dimension}

%\paragraph{Minimax Formulation.} The optimal bound of the nontrivial rounds given an error budget can be rephrased as the number of mistake made if both the learner and the adversary plays optimally. \cz{This is inaccurate since the constraints on $\hat{y}_i$'s haven't been written down.}
%\[ \max_{x_1} \min_{\hat{y}_1} \max_{y_1: \calH[(x_1,y_1)] \neq \emptyset} \ldots \max_{x_n} \min_{\hat{y}_n} \max_{y_n: \calH[(x_1,y_1),\ldots,(x_n,y_n)] \neq \emptyset} \sum_{t=1}^n I(\hat{y}_t = \bot) + I(\hat{y}_t = -y_t) \]

We now define extended Littlestone's Dimension, which measures the difficulty of online learning a hypothesis class in the $k$-SZB model.

\paragraph{Extended Mistake Trees.} An adversary's strategy in response to a deterministic learner in the $k$-SZB model can be succinctly represented by extended mistake trees. An extended mistake tree for $\calH$ is a full\footnote{A full binary tree is one in which every internal node has exactly two children.} binary tree, whose leaves are classifiers in $\calH$ and whose internal nodes are examples in $\calX$. An extended mistake tree may have no internal node, in which case it only contains a leaf corresponding to a classifier $h$ in $\calH$ -- we call it a zeroth order extended mistake tree. Unlike mistake trees, now, there are two type of edges: solid and dashed, representing mistakes and abstentions, respectively. Each node is associated with two downward solid edges, one to each child. Additionally, each node is associated with exactly one downward dashed edge connecting to one of its two children. For a downward edge of a node, whether solid or dashed, if it is connected with the node's left child, then it is labeled $-1$, and vice versa. Just as in mistake trees, for each leaf, the associated classifier agrees with the internal nodes and edges along the path up to the root. See Figure~\ref{fig:extendedmistaketree} for an illustration.
%\cz{A depth-$0$ extended mistake tree is one that has no internal nodes -- it has only one leaf corresponding to a classifier $h$ in $\calH$.}
%\kc{Are the left and right children always associated with labels $-1$ and $+1$ respectively? If so, this should be mentioned earlier while defining the mistake tree.}

A root to leaf path $p$ of an extended mistake tree $T$ is a sequence of nodes and edges $v_1e_1v_2e_2\ldots v_ne_nv_{n+1}$, where $v_1, \ldots, v_n$ are internal nodes in $T$ corresponding to examples in $\calX$, $v_1$ is the root node of $T$, each $e_i$ is an edge in $T$ that connects $v_i$ and $v_{i+1}$, $v_{n+1} = h$ is a classifier in $\calH$ corresponding to a leaf in $T$. Here if there are multiple edges between $v_i$ and $v_{i+1}$, any one of them can be used by $p$.

Given an extended mistake tree $T$, the associated adversarial strategy can be described as follows.
At $t=1$, the adversary chooses the example $x_1$ corresponding to the root node to show to the learner. If the learner predicts $\hat{y}_1 = -1$, it reveals label $y_1 = +1$, follows the downward solid edge labeled $+1$, and vice versa.  Otherwise, if $\hat{y}_1 = \bot$, it reveals $y_1$ as the label on the dashed edge and follows the downward dashed edge.
At time $t \geq 2$, if the adversary reaches a node with example $x_t$, then $x_t$ is shown to the learner, and one of its adjacent downward edges is followed.
The interaction comes to an end when a leaf is reached.
It can be seen that with this strategy, the adversary forces every round to be nontrivial. If the depth of the leaf reached is $d$, then the number of nontrivial rounds is $d$.

As an example, the extended mistake tree in Figure~\ref{fig:extendedmistaketree} can be used by the adversary as follows.
Initially $x_1 = 2$ is presented to the learner.
If the learner predicts $\hat{y}_1 = -1$, the adversary reveals label $y_1 = +1$ and follows the right downward solid edge to reach node $x_2 = 3$.
At time $t = 2$, the learner now shows example $x_2 = 3$;
If the learner predicts $\hat{y}_2 = \bot$, the adversary reveals $y_2 = +1$ according to the label on the dashed edge and follows the edge to reach node $x_3 = 4$.
At time $t = 3$, the learner shows example $x_3 = 4$;
If the learner predicts $\hat{y}_2 = +1$, the adversary reveals label $y_2 = -1$ and follows the left downward solid edge to reach a leaf containing hypothesis $h_3$.
This concludes the interaction, and the learner makes a total of 3 nontrivial rounds: 2 mistakes and 1 abstentions.
Note that realizability assumption is maintained, as $h_3 \in \calH$ agrees with the examples $(2,+1)$, $(3,+1)$, $(4,-1)$ shown.
More generally, one can show that if the learner is not allowed to make any mistakes, then the adversary is able to force 3 nontrivial rounds by following this strategy. This motivates the definition below.

\begin{definition}
We say that an extended mistake tree $T$ is $(k,m)$-difficult for integers $k, m \geq 0$, if all its root to leaf paths in $T$ using at most $k$ solid edges have length at least $m$.
\end{definition}

%\paragraph{Remark.} Concretely, a legal root to leaf path $p$ is a sequence of nodes and edges $v_1e_1v_2e_2\ldots v_ne_nv_{n+1}$, where $v_1, \ldots, v_n$ are internal nodes corresponding to instances in $\calX$, $v_1$ is the root node, $e_1, \ldots, e_n$ are edges in $T$, $v_{n+1} = h$ is a classifier in $\calH$ corresponding to a leaf. For each $i$, edge $e_i$ connects $v_i$ and $v_{i+1}$. The length of a path $l(p)$ is defined as the number of edges in $p$.

For example, the extended mistake tree in Figure~\ref{fig:extendedmistaketree} is $(0,3)$-difficult.
%As we see in Lemma~\ref{lem:emtusage}, If we can construct a $(k,d)$-difficult extended mistake tree, then the associated strategy forces any deterministic algorithm to make either $\geq k+1$ mistakes, or $\geq d$ abstentions.

%Fix $k$. If we can find a $(k,d)$-difficult extended mistake tree with large $d$, then the adversary is able to force more nontrivial rounds using the associated strategy. On the opposite side, if we cannot find a $(k,d)$-difficult extended mistake tree with a small $d$, then it is likely that the learner is able to achieve good performance. This intuition will be formalized by Theorem~\ref{thm:equivalence} below.

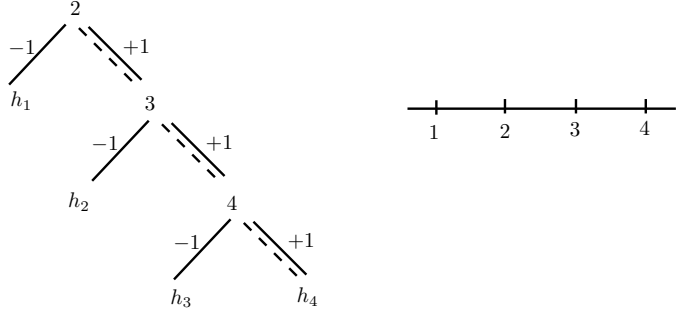
\begin{figure}
\centering
% Generated with LaTeXDraw 2.0.8
% Mon Jan 11 14:57:51 PST 2016
% \usepackage[usenames,dvipsnames]{pstricks}
% \usepackage{epsfig}
% \usepackage{pst-grad} % For gradients
% \usepackage{pst-plot} % For axes
\scalebox{0.8} % Change this value to rescale the drawing.
{
\begin{pspicture}(0,-2.6292188)(11.37,2.6292188)
\usefont{T1}{ppl}{m}{n}
\rput(2.6145313,0.84578127){$3$}
\usefont{T1}{ppl}{m}{n}
\rput(3.9745312,-0.81421876){$4$}
\usefont{T1}{ppl}{m}{n}
\rput(3.1345313,-2.3942187){$h_3$}
\usefont{T1}{ppl}{m}{n}
\rput(5.2345314,-2.3742187){$h_4$}
\usefont{T1}{ppl}{m}{n}
\rput(1.4345312,-0.7942188){$h_2$}
\usefont{T1}{ppl}{m}{n}
\rput(1.3745313,2.4257812){$2$}
\psline[linewidth=0.04cm](6.89,0.75578123)(11.35,0.75578123)
\psline[linewidth=0.04cm](7.37,0.8957813)(7.37,0.65578127)
\psline[linewidth=0.04cm](8.51,0.91578126)(8.51,0.67578125)
\psline[linewidth=0.04cm](9.69,0.91578126)(9.69,0.67578125)
\psline[linewidth=0.04cm](10.85,0.91578126)(10.85,0.67578125)
\usefont{T1}{ppl}{m}{n}
\rput(7.3345313,0.36578125){$1$}
\usefont{T1}{ppl}{m}{n}
\rput(8.514531,0.38578126){$2$}
\usefont{T1}{ppl}{m}{n}
\rput(9.674531,0.40578124){$3$}
\usefont{T1}{ppl}{m}{n}
\rput(10.834531,0.42578125){$4$}
\usefont{T1}{ppl}{m}{n}
\rput(0.47453126,0.8857812){$h_1$}
\psline[linewidth=0.04cm](1.21,2.1557813)(0.27,1.1557813)
\psline[linewidth=0.04cm](1.59,2.1157813)(2.47,1.2357812)
\usefont{T1}{ppl}{m}{n}
\rput(0.48453125,1.7657813){$-1$}
\usefont{T1}{ppl}{m}{n}
\rput(2.3845313,1.7857813){$+1$}
\psline[linewidth=0.04cm,linestyle=dashed,dash=0.16cm 0.16cm](1.43,2.0957813)(2.31,1.2157812)
\psline[linewidth=0.04cm](2.59,0.55578125)(1.65,-0.44421875)
\psline[linewidth=0.04cm](2.97,0.5157812)(3.85,-0.36421874)
\usefont{T1}{ppl}{m}{n}
\rput(1.8645313,0.16578124){$-1$}
\usefont{T1}{ppl}{m}{n}
\rput(3.7645311,0.18578126){$+1$}
\psline[linewidth=0.04cm,linestyle=dashed,dash=0.16cm 0.16cm](2.81,0.49578124)(3.69,-0.38421875)
\psline[linewidth=0.04cm](3.95,-1.0842187)(3.01,-2.0842187)
\psline[linewidth=0.04cm](4.33,-1.1242187)(5.21,-2.0042188)
\usefont{T1}{ppl}{m}{n}
\rput(3.2245312,-1.4742187){$-1$}
\usefont{T1}{ppl}{m}{n}
\rput(5.1245313,-1.4542187){$+1$}
\psline[linewidth=0.04cm,linestyle=dashed,dash=0.16cm 0.16cm](4.17,-1.1442188)(5.05,-2.0242188)
\end{pspicture} 
}
\caption{An extended mistake tree with respect to set of threshold classifiers $\calH = \cbr{h_i = 2 I(x \leq i) - 1: i =1,2,3,4}$.}
\label{fig:extendedmistaketree}
\end{figure}

\paragraph{Extended Standard Optimal Algorithm.} Algorithm~\ref{alg:esoa} presents the Extended Standard Optimal Algorithm (SOA.DK), which, as we will show, is an optimal deterministic algorithm for online prediction in the $k$-SZB model in the realizable case. Note that it works even when the hypothesis class $\calH$ is infinite.
Similar to the Standard Optimal Algorithm, it maintains a version space $V$. For a new example $x_t$, it predicts $\hat{y}_t \in \cbr{-1,+1,\bot}$ by computing function $\ELdim$ over subsets of $V$. The function $\ELdim$ is defined as follows.

\begin{definition}[Extended Littlestone's Dimension]
For a hypothesis class $V$ and integer $k \geq 0$, the extended Littlestone's dimension $\ELdim(V, k)$ is defined as:
\[ \ELdim(V,k) := \sup\cbr{m \in \N: \text{There exists a $(k,m)$-difficult extended mistake tree for $V$}}\]
\end{definition}
We remark that if for every integer $m$, $V$ has a $(k,m)$-extended mistake tree, then $\ELdim(V,k) = \infty$; If $V = \emptyset$, then $\ELdim(V,k) = -\infty$. Since for $k' < k$, a $(k,m)$-difficult extended mistake tree is also $(k',m)$-difficult, $\ELdim(V,k)$ is monotonically nonincreasing with respect to $k$.

We first show that when $\ELdim(V,k)$ is high, then an adversary can force a large number of nontrivial rounds by showing a $V$-realizable sequence, to any deterministic algorithms that guarantees at most $k$ mistakes.
\begin{lemma}
Suppose we are given a hypothesis set $V$ and integers $k \geq 0, m \geq 0$. If $\ELdim(V,k) \geq m$, then there is a strategy of the adversary that presents a $V$-realizable sequence and that can force any deterministic algorithm that guarantees $\leq k$ mistakes to have $\geq m$ nontrivial rounds.
\label{lem:emtusage}
\end{lemma}

In the following lemma, we show that given a mistake budget $k$, if the extended Littlestone's dimension of $V$ is small, then SOA.DK has a small number of nontrivial rounds for $V$-realizable sequences.
\begin{lemma}[Performance Guarantees of SOA.DK]
Suppose we are given a hypothesis class $V$ and integers $k \geq 0, m \geq 0$. If $\ELdim(V,k) \leq m$, then Algorithm~\ref{alg:esoa}, when run on $V$ with mistake budget $k$, achieves a $(k,m)$-SZB bound with respect to any adversary that shows $V$-realizable sequences.
\label{lem:perfguar}
\end{lemma}
%\cz{We need more precise definition on "when run on $V$", and $V$-realizable}

%We now name the function $K(\cdot, \cdot)$ as $\ELdim(\cdot, \cdot)$ (Extended Littlestone's Dimension), as it characterizes the minimax number of nontrivial rounds given a mistake budget and a hypothesis class in realizable case.

\begin{algorithm}[h]
\caption{Extended Standard Optimal Algorithm: SOA.DK}
\begin{algorithmic}[1]
\STATE Input: hypothesis class $\calH$, mistake budget $k$.
\STATE Initialize version space $V \gets \calH$.
\FOR{$t = 1,2,\ldots,$}
\STATE Receive example $x_t \in \calX$.
\IF[All classifiers in $V$ predict unanimously]{$x_t \in \DIS(V)$}
    \STATE Predict $\hat{y}_t = h(x_t)$, where $h$ is an arbitrary hypothesis in $V$.
\ELSE[There is disagreement among $V$]
    \IF[Zero mistake budget, must output $\bot$]{$k = 0$}
        \STATE Predict $\hat{y}_t = \bot$.
    \ELSE[Predict by minimizing the $\ELdim$ of future version space]
        \STATE Compute $m_{+1} = \ELdim(V[(x_t,-1)], k-1)$, $m_{-1} = \ELdim(V[(x_t,+1)], k-1)$, and $m_{\bot} = \max(\ELdim(V[(x_t,-1)], k), \ELdim(V[(x_t,+1)], k))$
        \STATE Predict $\hat{y}_t = \arg\min\cbr{m_y: y \in \cbr{-1,+1,\bot}} $.
    \ENDIF
\ENDIF
\STATE Receive label $y_t$.
\IIF{$\hat{y}_t = -y_t$ or $\hat{y}_t = \bot$} $V \gets V[(x_t, y_t)]$ \ENDIIF \COMMENT{Update version space}
\IIF{$\hat{y}_t = -y_t$} $k \gets k - 1$ \ENDIIF \COMMENT{Update mistake budget}

\ENDFOR
\end{algorithmic}
\label{alg:esoa}
\end{algorithm}

%We show the following connection between extended mistake trees and optimal number of nontrivial rounds in $k$-SZB model.

An immediate consequence of Lemma~\ref{lem:perfguar} is that SOA.DK is optimal, in the sense that it has the smallest number of worst case nontrivial rounds, amongst all {\em deterministic} algorithms that work in $k$-SZB model.

\begin{theorem}[Optimality of SOA.DK]
Suppose we are given a hypothesis class $\calH$ and integers $k \geq 0$, $m \geq 1$ such that $\ELdim(\calH, k) = m$. Then:
\begin{enumerate}[(a)]
\item SOA.DK achieves a $(k,m)$-SZB bound for any adversary that shows $\calH$-realizable sequences.
\item There exists an adversary showing $\calH$-realizable sequences, such that no deterministic algorithm $\calA$ can achieve a $(k,m-1)$-SZB bound.
\end{enumerate}
\label{thm:optimality}
\end{theorem}

The following simple property relates extended Littlestone's dimension to Littlestone's dimension.

\begin{theorem}[Relating $\Ldim$ to $\ELdim$]
Suppose we are given a hypothesis class $\calH$. If $\Ldim(\calH) = d < \infty$, then
\[ \ELdim(\calH, d) = d \]
\label{thm:eldimldim}
%Thus,
%\[ \min\cbr{l: \ELdim(\calH, l) \leq l} \leq \Ldim(\calH) \leq \max\cbr{l: \ELdim(\calH, l) \geq l}\]
\end{theorem}

%\paragraph{Extended Littlestone's Dimension and its Recursive Definition.} Consider the following formulation of the online prediction problem: minimize the number of abstentions, subject to that the number of mistakes is at most $k$. We provide a characterization of optimal number of nontrivial rounds in the realizable case. The optimal bound of number of nontrivial rounds with respect to $\calH$ and mistake budget $k$, called Extended Littlestone's dimension $\ELdim(\calH,k)$, is defined recursively as:

\paragraph{Recursive Definition.} We provide a recursive characterization of Extended Littlestone's dimension. For finite $\calH$\footnote{Just as with Littlestone's dimension, for infinite $\calH$, the recurrence may not reach the base case.}, the following recurrence holds for its extended Littlestone's dimension:
\begin{eqnarray*}
&&\ELdim(\calH, k) = \\
&&\begin{cases} 0 & |\calH| = 1
\\ 1 + \max_{x \in \DIS(\calH)} \max_{y \in \cbr{\pm 1}} \ELdim(\calH[(x,y)], 0) & |\calH| > 1, k = 0 \\
1 + \max_{x \in \DIS(\calH)} \max_{y \in \cbr{\pm 1}} \min (\ELdim(\calH[(x,y)], k-1), \ELdim(\calH[(x,-y)], k)) & |\calH| > 1, k \geq 1  \end{cases}
\end{eqnarray*}
The recurrence is an immediate consequence of Lemma~\ref{lem:eldimrecursive} in Appendix~\ref{sec:pfrealizable}.
%\cz{This definition is tricky, since some recursion may not even reach the base case! But it is an intuitive characterization.}
%\begin{eqnarray*}
%\ELdim(\calH, 0) = \begin{cases} \infty & |\calH| = \infty \\ |\calH|-1 & |\calH| < \infty \end{cases}
%\end{eqnarray*}

%\begin{eqnarray*}
%&&\ELdim(\calH, k) = \\
%&&\begin{cases} 0 & |\calH| = 1 \\ 1 + \max_{x \in \DIS(\calH)} \max \cbr{\ELdim(\calH[(x,-1)], 0), \ELdim(\calH[(x,+1)], 0)} & |\calH| > 1, k = 0 \\ 1 + \max_{x \in \DIS(\calH)} \min \{\ELdim(\calH[(x,-1)], k-1), \ELdim(\calH[(x,+1)], k-1),\\ \max(\ELdim(\calH[(x,-1)], k), \ELdim(\calH[(x,+1)], k)) \} & |\calH| > 1, k \geq 1  \end{cases}
%\end{eqnarray*}

\section{Properties of Extended Littlestone's Dimension}
\label{sec:eldimension}

%In this section, we develop techniques for analyzing the extended Littlestone's dimension of a hypothesis class $\calH$. In subsection~\ref{sub:sgf}, a capacity measure on $\calH$ named tree shattering coefficient are presented. Next, in subsection~\ref{sub:upperbound}, we present Theorem~\ref{thm:growth} to bound $\calH$'s extended Littlestone's dimension in terms of its tree shattering coefficient. 

We next present upper bounds on the Extended Littlestone's Dimension of a hypothesis class $\calH$. Our upper bounds depend on the {\em{tree shattering coefficient}}, a notion analogous to the growth function, which is implicit in~\cite{BPS09}. We also present some examples of Extended Littlestone's Dimension.

\subsection{Tree Shattering Coefficient}
\label{sub:sgf}
The shattering coefficient (also known as the growth function), initially studied in~\citep{VC71}, is a key notion in PAC learnability.
\begin{definition} 
Given a hypothesis $\calH$, the shattering coefficient of $\calH$, $\Pi(\calH,t)$ is defined as the maximum number of labelings achievable by $\calH$ over $t$ points. Formally,
\[ \Pi(\calH,t) := \max_{x_1, \ldots, x_t} |\cbr{ (h(x_1), \ldots, h(x_t)): h \in \calH }|\]
\end{definition}
%As the number of possible labeling among $t$ points is at most $2^t$, $\Pi(\calH,t) \leq 2^t$. Moreover, it is shown in~\citep{VC71} that if $\calH$ is of VC dimension $d$, then $\Pi(\calH,t) = O(t^d)$.

Inspired by the shattering coefficient, in online learning, we define the notion of {\em tree shattering coefficient} below, implicit in~\citep{BPS09}. As we will see, this notion is crucial to online learnability in both the mistake bound and the $k$-SZB models.
First we set up our notation by adopting the notion of trees in~\citep{RST10}. 
\begin{definition}[$\calX$-valued Trees, see~\citep{RST10}]
A depth-$t$ $\calX$-valued tree $\x$ is a series of mappings $(\x_1, \ldots, \x_t)$, where $\x_i: \cbr{\pm 1}^{i-1} \to \calX$. The root of the tree $\x$ is the constant function $\x_1 \in \calX$. For integer $t$, the mapping $\x_t(\cdot)$ takes care of the nodes in level $t$.
\end{definition}

To see why a series of mappings corresponds to a tree, we first note that a tuple $(\epsilon_1, \ldots, \epsilon_{s-1})$ in $\cbr{\pm 1}^{s-1}$ can be thought of as a left/right sequence of length $s-1$, where $-1$ stands for left and $+1$ stands for right, respectively. The node reached from the root following the path corresponding to the left/right sequence corresponds to $\x_s(\epsilon_1, \ldots, \epsilon_{s-1}) \in \calX$.
For example, the root node corresponds to $\x_1 \in \calX$, the left child of the root corresponds to $\x_2(-1) \in \calX$, the right child of the left child of the root corresponds to $\x_3(-1,+1) \in \calX$, etc. See Figure~\ref{fig:tree} for an illustration. We slightly abuse the notation to let $\x_t(\epsilon)$ denote $\x_t(\epsilon_1, \ldots, \epsilon_{t-1})$. 
%More generally, if a node can be reached from root by a left/right sequence $(\epsilon_1, \ldots, \epsilon_{s-1}) \in \cbr{\pm 1}^{s-1}$(where $-1$ represents left and $+1$ represents right, respectively), then the value of the node is $\x_s(\epsilon_1, \ldots, \epsilon_{s-1})$. 

\begin{figure}
\centering
% Generated with LaTeXDraw 2.0.8
% Wed Feb 10 11:01:46 PST 2016
% \usepackage[usenames,dvipsnames]{pstricks}
% \usepackage{epsfig}
% \usepackage{pst-grad} % For gradients
% \usepackage{pst-plot} % For axes
\scalebox{1} % Change this value to rescale the drawing.
{
\begin{pspicture}(0,-1.8492187)(9.169063,1.8492187)
\usefont{T1}{ppl}{m}{n}
\rput(4.2545314,1.6457813){$\x_1$}
\psline[linewidth=0.04cm](4.03,1.2957813)(3.09,0.29578125)
\psline[linewidth=0.04cm](4.43,1.3157812)(5.39,0.31578124)
\usefont{T1}{ppl}{m}{n}
\rput(3.3045313,0.90578127){$-1$}
\usefont{T1}{ppl}{m}{n}
\rput(5.204531,0.92578125){$+1$}
\usefont{T1}{ppl}{m}{n}
\rput(3.0145311,0.01578125){$\x_2(-1)$}
\usefont{T1}{ppl}{m}{n}
\rput(5.974531,0.01578125){$\x_2(+1)$}
\usefont{T1}{ppl}{m}{n}
\rput(1.0945313,-1.6142187){$\x_3(-1,-1)$}
\usefont{T1}{ppl}{m}{n}
\rput(3.5345314,-1.6142187){$\x_3(-1,+1)$}
\usefont{T1}{ppl}{m}{n}
\rput(5.434531,-1.6142187){$\x_3(+1,-1)$}
\usefont{T1}{ppl}{m}{n}
\rput(7.7545314,-1.6142187){$\x_3(+1, +1)$}
\psline[linewidth=0.04cm](2.63,-0.36421874)(1.69,-1.3642187)
\psline[linewidth=0.04cm](3.03,-0.34421876)(3.99,-1.3442187)
\usefont{T1}{ppl}{m}{n}
\rput(1.9045312,-0.75421876){$-1$}
\usefont{T1}{ppl}{m}{n}
\rput(3.8045313,-0.7342188){$+1$}
\psline[linewidth=0.04cm](5.61,-0.32421875)(4.67,-1.3242188)
\psline[linewidth=0.04cm](6.01,-0.30421874)(6.97,-1.3042188)
\usefont{T1}{ppl}{m}{n}
\rput(4.884531,-0.71421874){$-1$}
\usefont{T1}{ppl}{m}{n}
\rput(6.784531,-0.69421875){$+1$}
\end{pspicture} 
}
\caption{A depth-$3$ $\calX$-valued tree $\x$.}
\label{fig:tree}
\end{figure}
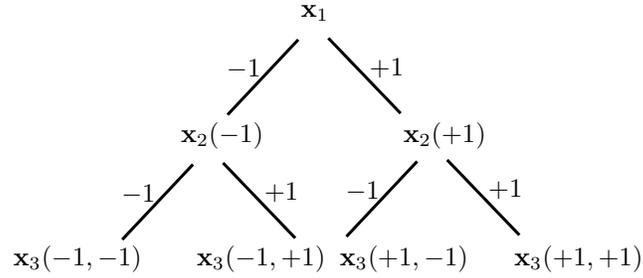

Note that a $\calX$-valued tree is not a mistake tree or an extended mistake tree, since it does not have leaves corresponding to hypotheses in $\calH$. 

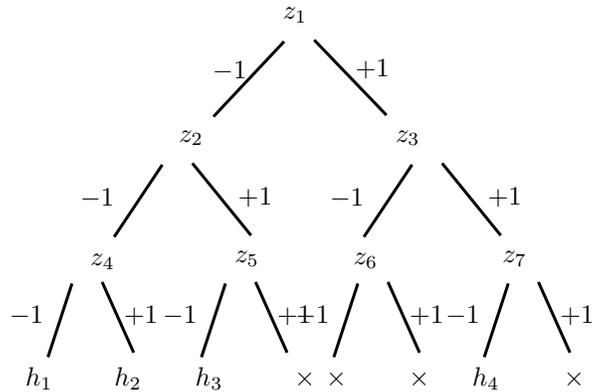
\begin{figure}
\centering
% Generated with LaTeXDraw 2.0.8
% Sun Apr 10 01:26:58 PDT 2016
% \usepackage[usenames,dvipsnames]{pstricks}
% \usepackage{epsfig}
% \usepackage{pst-grad} % For gradients
% \usepackage{pst-plot} % For axes
\scalebox{1} % Change this value to rescale the drawing.
{
\begin{pspicture}(0,-2.6332421)(8.469063,2.6332421)
\usefont{T1}{ppl}{m}{n}
\rput(3.9545312,2.4298048){$z_1$}
\psline[linewidth=0.04cm](3.81,2.0798047)(2.87,1.0798047)
\psline[linewidth=0.04cm](4.21,2.0998046)(5.17,1.0998046)
\usefont{T1}{ppl}{m}{n}
\rput(3.0845313,1.6898047){$-1$}
\usefont{T1}{ppl}{m}{n}
\rput(4.9845314,1.7098047){$+1$}
\usefont{T1}{ppl}{m}{n}
\rput(2.5745313,0.7998047){$z_2$}
\usefont{T1}{ppl}{m}{n}
\rput(5.454531,0.7998047){$z_3$}
\usefont{T1}{ppl}{m}{n}
\rput(1.3945312,-0.8501953){$z_4$}
\usefont{T1}{ppl}{m}{n}
\rput(3.3145313,-0.8101953){$z_5$}
\usefont{T1}{ppl}{m}{n}
\rput(4.8945312,-0.8301953){$z_6$}
\usefont{T1}{ppl}{m}{n}
\rput(6.8745313,-0.8101953){$z_7$}
\psline[linewidth=0.04cm](0.99,-1.1401954)(0.67,-2.1201954)
\psline[linewidth=0.04cm](1.39,-1.1201953)(1.81,-2.0601952)
\usefont{T1}{ppl}{m}{n}
\rput(0.38453126,-1.5701953){$-1$}
\usefont{T1}{ppl}{m}{n}
\rput(1.9045312,-1.5701953){$+1$}
\usefont{T1}{ppl}{m}{n}
\rput(0.5545313,-2.3982422){$h_1$}
\usefont{T1}{ppl}{m}{n}
\rput(1.7345313,-2.3982422){$h_2$}
\usefont{T1}{ppl}{m}{n}
\rput(2.8145313,-2.3982422){$h_3$}
\usefont{T1}{ppl}{m}{n}
\rput(6.494531,-2.3982422){$h_4$}
\usefont{T1}{ppl}{m}{n}
\rput(4.0845313,-2.414336){$\times$}
\usefont{T1}{ppl}{m}{n}
\rput(4.5045314,-2.414336){$\times$}
\usefont{T1}{ppl}{m}{n}
\rput(5.6045313,-2.414336){$\times$}
\usefont{T1}{ppl}{m}{n}
\rput(7.664531,-2.414336){$\times$}
\psline[linewidth=0.04cm](2.19,0.43980467)(1.55,-0.5201953)
\psline[linewidth=0.04cm](2.59,0.45980468)(3.37,-0.5001953)
\usefont{T1}{ppl}{m}{n}
\rput(1.3245312,-0.010195312){$-1$}
\usefont{T1}{ppl}{m}{n}
\rput(3.4245312,0.009804687){$+1$}
\psline[linewidth=0.04cm](5.51,0.43980467)(4.87,-0.5201953)
\psline[linewidth=0.04cm](5.91,0.45980468)(6.69,-0.5001953)
\usefont{T1}{ppl}{m}{n}
\rput(4.6445312,-0.010195312){$-1$}
\usefont{T1}{ppl}{m}{n}
\rput(6.744531,0.009804687){$+1$}
\psline[linewidth=0.04cm](3.03,-1.1401954)(2.71,-2.1201954)
\psline[linewidth=0.04cm](3.43,-1.1201953)(3.85,-2.0601952)
\usefont{T1}{ppl}{m}{n}
\rput(2.4245312,-1.5701953){$-1$}
\usefont{T1}{ppl}{m}{n}
\rput(3.9445312,-1.5701953){$+1$}
\psline[linewidth=0.04cm](4.79,-1.1401954)(4.47,-2.1201954)
\psline[linewidth=0.04cm](5.19,-1.1201953)(5.61,-2.0601952)
\usefont{T1}{ppl}{m}{n}
\rput(4.184531,-1.5701953){$-1$}
\usefont{T1}{ppl}{m}{n}
\rput(5.704531,-1.5701953){$+1$}
\psline[linewidth=0.04cm](6.79,-1.1401954)(6.47,-2.1201954)
\psline[linewidth=0.04cm](7.19,-1.1201953)(7.61,-2.0601952)
\usefont{T1}{ppl}{m}{n}
\rput(6.184531,-1.5701953){$-1$}
\usefont{T1}{ppl}{m}{n}
\rput(7.704531,-1.5701953){$+1$}
\end{pspicture} 
}
\caption{A concrete depth-3 $\calX$-valued tree $\x$, where $\x_1 = z_1$, $\x_2(-1) = z_2$, $\x_2(+1) = z_3$, $\x_3(-1,-1) = z_4$, $\x_3(-1,+1) = z_5$, $\x_3(+1,-1) = z_6$, $\x_3(+1,+1) = z_7$. There are 4 root to leaf paths that agrees with some hypothesis in $\calH$ ($\times$ in a leaf indicates that no hypothesis in $\calH$ agree with the path from root to it), i.e. $|S(\calH, \x)| = 4$.}
\label{fig:tree_cons}
\end{figure}

Given a $\calX$-valued tree $\x$, we add an extra level of edges at the bottom. Specifically for each leaf $\x_t(\epsilon)$, we attach a left and a right downward edge onto it, labeled $-1$ and $+1$ respectively. Now, consider every root to leaf path in the tree. If there is some classifier $h$ in $\calH$ that agrees with the path, we label the leaf with $h$; otherwise we label the leaf with symbol $\times$. We count the number of leaves not labeled $\times$, denoted by function $S(\calH, \x)$. See Figure~\ref{fig:tree_cons} for an example.

\begin{definition}
For a depth-$t$ $\calX$-valued tree $\x$, and a hypothesis class $\calH$, define function $S(\calH, \x)$ as the maximum number of labelings achievable by $\calH$ on $\x$. Formally,
\[ S(\calH, \x) := \cbr{ (\epsilon_1, \epsilon_2, \ldots, \epsilon_t) \in \cbr{\pm 1}^t: \epsilon_1 = h(\x_1(\epsilon)), \epsilon_2 = h(\x_2(\epsilon)), \ldots, \epsilon_t = h(\x_t(\epsilon)), \text{ for some } h \in \calH } \]
\end{definition}
%as the number of root to leaf paths in $\x$ that agrees with some hypothesis in $\calH$. Formally

\begin{definition} 
Given hypothesis class $\calH$, and integer $t \geq 1$, the tree shattering coefficient of $\calH$, $\calS(\calH,t)$ is defined as the maximum number of labelings achievable by $\calH$ over all depth-$t$ trees. Formally, 
\[\calS(\calH,t) := \max_{\x}  |S(\calH, \x)| \]
%Additionally define $\calS(\calH,0):=0$ if $\calH = \emptyset$, and $\calS(\calH,0):=1$ if $\calH \neq \emptyset$.
Additionally, define $\calS(\calH, 0) := 1$ if $\calH$ is nonempty, $\calS(\calH, 0) := 0$ if $\calH$ is empty.
\end{definition}

In other words, given hypothesis class $\calH$ and a depth-$t$ tree $\x$ with internal nodes only, there are at most $\calS(\calH,t)$ distinct paths in $T$ consistent with some classifier $h \in \calH$. Since a depth-$t$ has at most $2^t$ root to leaf paths, $\calS(\calH, t) \leq 2^t$. Note that if $\calH$ has a depth-$t$ mistake tree, then $\calS(\calH, t) = 2^t$. If we constrain the trees chosen to be constant among nodes in the same depth, then the tree shattering coefficient is equivalent to the shattering coefficient. In Appendix~\ref{sec:tscsgf}, we show that the tree shattering coefficient is connected with the sequential growth function(maximal sequential zero covering number), defined in~\citep{RST10}.

The following two lemmas give bounds on tree shattering coefficients, implicit in~\cite{RST10, BPS09}. For finite hypothesis class $\calH$, its tree shattering coefficient is at most the size of $|\calH|$.
\begin{lemma}
If $\calH$ is finite, then for any $t \geq 0$, $\calS(\calH,t) \leq |\calH|$.
\label{lem:finitegrowth}
\end{lemma}

Furthermore, if an infinite hypothesis class $\calH$ has Littlestone's dimension $d < \infty$, its tree shattering coefficient is polynomial in $t$, that is, $O(t^d)$.

\begin{lemma}
If $\calH$ has Littlestone's dimension $d < \infty$, then for any $t \geq 0$, $\calS(\calH,t) \leq \binom{t}{\leq d}$.
\label{lem:ldimgrowth}
\end{lemma}

\subsection{Upper Bound on Extended Littlestone's Dimension}
\label{sub:upperbound}
We present Theorem~\ref{thm:growth}, the main result of this section, which upper bounds the extended Littlestone's dimension in terms of tree shattering coefficient. Intuitively, if $\calH$ is not expressive, then it has small tree shattering coefficient, and a tighter upper bound on its extended Littlestone's dimension can be established. Note that the bound is valid even if $\calH$ is infinite, and hence it is a strict generalization of~\cite{SZB10}. 
%See Theorem~\ref{thm:usingletons}, Lemmas~\ref{lem:finiteub},~\ref{lem:infiniteub} for examples.
\begin{theorem}
For any hypothesis class $\calH$ and integer $k \geq 0$,
\[ \ELdim(\calH, k) \leq \sup\left\{t: \binom{t}{\leq k+1} \leq \calS(\calH,t)\right\}\]
\label{thm:growth}
\end{theorem}

For finite hypothesis classes one has the following corollary. 

\begin{corollary}
For any hypothesis class $\calH$ such that $|\calH|<\infty$ and integer $k \geq 0$,
\[ \ELdim(\calH, k) \leq \max\left\{t: \binom{t}{\leq k+1} \leq |\calH|\right\} \]
\label{cor:finiteeldim}
\end{corollary}

Since $\binom{t}{\leq k+1} \geq (\frac{t}{k+1})^{k+1}$, this implies $\ELdim(\calH, k) \leq \max\left\{t: (\frac{t}{k+1})^{k+1} \leq |\calH|\right\} \leq (k+1) |\calH|^{\frac{1}{k+1}}$, which recovers the result of~\cite{SZB10}.~\footnote{Although it is implicit in~\cite{SZB10} that the result can be refined by using the optimal solution of the Egg Dropping Game~\citep{GF08,B04}, here we give a alternative proof using a more general technique.}

\subsection{Case Study: Thresholds (Finite Class)} 
We give a precise characterization of the Extended Littlestone's dimension for the class of thresholds. In this case, the bound given by Theorem~\ref{thm:growth} is tight.

Consider the instance domain $\calX$ being $\R$ and the hypothesis class $\calH$ being the set of $n$ distinct threshold functions $\cbr{2I(x \leq t)-1: t \in \cbr{t_1, \ldots, t_n}}$.~\footnote{Note that for an infinite set of thresholds, e.g. $\calH = \cbr{2I(x \leq t)-1: t \in [0,1]}$, $\Ldim(\calH) = \infty$, hence $\ELdim(\calH, k) = \infty$ for any finite $k$.}
\begin{theorem}
Consider $\calH$ a set of threshold classifiers $\calH = \cbr{2I(x \leq t)-1: t \in \cbr{t_1, \ldots, t_n}}$. Then
\[ \ELdim(\calH, k) = \max \left\{t: \binom{t}{\leq k+1} \leq n\right\} \]
\label{thm:thresholds}
\end{theorem}
The proof of Theorem~\ref{thm:thresholds} is provided in Appendix~\ref{sec:pfeldimension}. The upper bound follows immediately from Corollary~\ref{cor:finiteeldim}. The lower bound comes from an explicit construction of optimal extended mistake trees by exploiting structure in the class of threshold classifiers.

\subsection{Case Study: Union of Singletons (Infinite Class)}
We give a precise characterization of the Extended Littlestone's dimension for the class of unions of singletons. In this case the bound given by Theorem~\ref{thm:growth} is tight.
Consider the concept class of union of at most $l$ singletons $\calC^l$, with instance domain $\calX$ such that $|\calX| = \infty$. Note that $\Ldim(\calC^l) = l$ and $\calS(\calC^l,t) = \binom{t}{\leq l}$(See Lemma~\ref{lem:tscus} for a proof). We have the following result.
\begin{theorem}
Consider the hypothesis class $\calC^l$, the class of union of at most $l$ singletons. Then,
\[ \ELdim(\calC^l, k) = \sup \left\{t: \binom{t}{\leq k+1} \leq \binom{t}{\leq l}\right\} = \begin{cases} \infty, & k \leq l-1 \\ l, & k \geq l \end{cases} \]
\label{thm:usingletons}
\end{theorem}
Note that Theorem~\ref{thm:usingletons} involves infinite hypothesis classes and is broader than the results of~\cite{SZB10}.
The proof of Theorem~\ref{thm:usingletons} is provided in Appendix~\ref{sec:pfeldimension}. The upper bound follows immediately from Theorem~\ref{thm:growth}. The lower bound comes from an explicit construction of optimal extended mistake trees by exploiting structures in the class of union of singleton classifiers.

\section{Non-Realizable Case}
\label{sec:nonrealizable}

We now consider the non-realizable case. For the rest of the section, we assume the $l$-bias assumption holds, i.e. the sequence $(x_1,y_1), \ldots, (x_n, y_n)$ presented by the adversary is $\calH^l$-realizable. Recall that $\calH^l = \calH \cdot \calC^l$, the class of hypothesis that disagrees with $\calH$ on at most $l$ points.

%(1) $x_1, \ldots, x_n$ are distinct. (2) there is a hypothesis $h \in \calH$ such that $|\cbr{t: h(x_t) \neq y_t}| \leq l$.
%\cz{If we use the reduction $\calH \to \calH^l$, then we can use construction of mistake tree to unify the presentation.}

\subsection{Lower Bounds for Deterministic Prediction}
\label{sec:lbnonrealizable}
A natural question is, when the $l$-bias assumption holds, is it possible to derive algorithms with a small number of abstentions in $k$-SZB model? Perhaps surprisingly, the answer depends on whether $\calH$ is finite or not. We show next that there is a finite hypothesis class $\calH$, such that for any $k < l$, and any integer $m$, any algorithm which is guaranteed to make $k$ or less mistakes can be forced to abstain at least $m$ times. Moreover, for any $m$, there is a infinite hypothesis class with Littlestone's dimension $d$, such that for any $k < l + d$, any algorithm that is guaranteed to make $k$ or less mistakes can be forced to abstain at least $m$ times.

\subsubsection{Lower Bounds for Finite Hypothesis Classes}

We first show that, for finite hypothesis classes $\calH$, when $k < l$, no algorithm can guarantee a $(k,m)$-SZB bound with finite $m$ under the $l$-bias assumption.
\begin{theorem}
There exists an instance domain $\calX$, a single-element hypothesis class $\calH$, such that the following holds. If $k < l$, then for any integer $m \geq 0$, there exists a strategy of the adversary satisfying the $l$-bias assumption that forces any deterministic algorithm guaranteeing at most $k$ mistakes to have at least $m+1$ nontrivial rounds.
\label{thm:finitelb}
\end{theorem}

\subsubsection{Lower Bounds for Infinite Hypothesis Classes}

We show that, given a hypothesis classes $\calH$ with $\Ldim(\calH) = d$, when $k < l + d$, no algorithm can guarantee a $(k,m)$-SZB bound for finite $m$ under the $l$-bias assumption.

\begin{theorem}
There exists an instance domain $\calX$, a hypothesis class $\calH$ with Littlestone's dimension $d < \infty$, such that the following holds. If $k < l + d$, then for any integer $m \geq 0$, there exists a strategy of the adversary satisfying the $l$-bias assumption that forces any deterministic algorithm guaranteeing at most $k$ mistakes to have at least $m+1$ nontrivial rounds.
\label{thm:infinitelb}
\end{theorem}

\subsection{Upper Bounds}
%Recall that $\calC^l$ the class of unions of at most $l$ singletons. That is,
%\[ \calC^l = \cbr{I(x = y_1 \vee \ldots \vee x = y_d), y_1: \ldots, y_d \in \calX, d \leq l} \]
%\cz{we will use the reduction $\calH \to \calH^l$ to unify the presentation.}

\subsubsection{Upper Bounds for Finite Hypothesis classes}

Since a sequence satisfying the $l$-bias assumption is $\calH^{l}$-realizable, to provide an upper bound on the number of non-trivial rounds under this assumption, we need to provide an upper bound on $\ELdim(\calH^l, k)$. We now provide such upper bounds on arbitrary finite hypothesis classses $\calH$. Note that since the hypothesis class $\calH^l$ is infinite, this result is more general than the kind of results in~\citep{SZB10}.
\begin{lemma}
Suppose we are given a finite hypothesis class $\calH$, integer $k \geq 0$, $l \geq 0$ such that $k \geq l$. Then,
\[ \ELdim(\calH^l, k) \leq e(k+1) \cdot |\calH|^{\frac{1}{k+1-l}} \]
\label{lem:finiteub}
\end{lemma}

\begin{corollary}
Suppose we are given a finite hypothesis class $\calH$ and integers $k,l \geq 0$ such that $k \geq l$. If Algorithm SOA.DK is run with input hypothesis class $\calH^l$ and mistake budget $k$, then for any adversary that shows sequences satisfying the $l$-bias assumption with respect to $\calH$, SOA.DK makes at most $k$ mistakes and has at most $ e(k+1) \cdot |\calH|^{\frac{1}{k+1-l}} $ nontrivial rounds.
\label{cor:finiteub}
\end{corollary}

\subsubsection{Upper Bounds for Infinite Hypothesis classes}
%For hypothesis class $\calH$ with finite Littlestone's dimension $d$, we give bounds on $\ELdim(\calH^l, k)$ for $k \geq l+d$, by invoking Theorem~\ref{thm:growth}. Note that the result is not a consequence of~\citep{SZB10}, since the hypothesis class $\calH^l$ is infinite.
We now derive a corresponding upper bound for infinite hypothesis classes $\calH$ with finite Littlestone's dimension.

\begin{lemma}
Suppose we are given a hypothesis class $\calH$ with Littlestone's dimension $d < \infty$, integer $k \geq 0$, $l \geq 0$ such that $k \geq l + d$. Then,
\[ \ELdim(\calH^l, k) \leq (k+1) \cdot e^{\frac{2k+2}{k+1-l-d}} \]
\label{lem:infiniteub}
\end{lemma}

\begin{corollary}
Suppose we are given a hypothesis class $\calH$ with Littlestone's dimension $d$ and integer $k,l \geq 0$ such that $k \geq l + d$. If Algorithm SOA.DK is run with input hypothesis class $\calH^l$ and mistake budget $k$, then for any adversary that shows sequences satisfying the $l$-bias assumption with respect to $\calH$, SOA.DK makes at most $k$ mistakes and has at most $ (k+1) \cdot e^{\frac{2k+2}{k+1-l-d}} $ nontrivial rounds.
\label{cor:infiniteub}
\end{corollary}

%\subsection{Proof Technique: Reduction to Realizable Case}
%We show that given a sequence $(x_1,y_1), \ldots, (x_n, y_n)$, where $x_1, \ldots, x_n$ are distinct, $l$-bias assumption with respect to $\calH$ is equivalent to realizability with respect to $\calH^l$.
%\begin{lemma}
%Given hypothesis class $\calH$ and a sequence $S = ((x_1, y_1), \ldots, (x_n, y_n))$, where $x_1, \ldots, x_n$ are distinct. The following are equivalent: 

%(1) $S$ satisfies $l$-bias assumption with respect to $\calH$.

%(2) $S$ is realizable with respect to $\calH^l$.
%\label{lem:reduction}
%\end{lemma}
%\begin{proof}
%($\Rightarrow$) If sequence $S = ((x_1, y_1), \ldots, (x_n, y_n))$ satisfies $l$-bias assumption with respect to $\calH$, then there exists a hypthesis $h$ in $\calH$ such that
%\[ |\cbr{t: h(x_t) \neq y_t }| \leq l\]
%We enumerate the elements in $\cbr{t: h(x_t) \neq y_t }$ as $z_1, \ldots, z_d$, where $d \leq l$. Define hypothesis
%\[ h^l := h \cdot I(x = z_1 \vee \ldots x = z_d) \]
%It can be seen from construction that $h^l \in \calH^l$, and $h^l$ classifies sequence $(x_1, y_1), \ldots, (x_n, y_n)$ correctly.

%($\Leftarrow$) Suppose $S = ((x_1, y_1), \ldots, (x_n, y_n))$ is realizable with respect to $\calH^l$. That is, there is some $h \in \calH$ and $z_1,\ldots,z_d \in \calX$, $d \leq l$ such that 
%$h^l := h \cdot I(x = z_1 \vee \ldots x = z_d)$ is consistent with $S$.
%It can seen that 
%\[ \cbr{t: h(x_t) \neq y_t } \subseteq \cbr{z_1, \ldots, z_d} \]
%hence, $|\cbr{t: h(x_t) \neq y_t}| = d \leq l$.
%\end{proof}

\subsection{Lower Bounds for Randomized Prediction}

We show that the results in Section~\ref{sec:lbnonrealizable} hold even when the learner makes soft predictions.

\paragraph{Randomized Prediction Model.} 
Consider the following randomized variant of online classification model. At time $t$, the adversary presents example $x_t$ in $\calX$, and the learner outputs a tuple $(p_{t,-}, p_{t,+}, 1 - p_{t,-} - p_{t,+})$, with $p_{t,-} \geq 0$, $p_{t,+} \geq 0$ and $1 - p_{t,-} - p_{t,+} \geq 0$. The tuple $(p_{t,-}, p_{t,+}, 1 - p_{t,-} - p_{t,+})$ represents the learner's strategy of predicting $+1$ with probability $p_{t, +}$, $-1$ with probability $p_{t, -}$ and abstaining with probability $1 - p_{t, +} - p_{t, -}$. The adversary then reveals an outcome $y_t \in \cbr{-1,+1}$, and the learner incurs a mistake penalty of $p_{t,+}$ if $y_t = -1$ and $p_{t,-}$ if $y_t = 1$; it also incurs an abstention penalty of $1 - p_{t,+} - p_{t,-}$. When $p_{t,+}$ and $p_{t,-}$ take values in $\cbr{0,1}$, observe that this is equivalent to our prediction model in Section~\ref{sec:setting}.
%\kc{sounds weird} 

%\kc{need to specify what is the penalty of the learner}
%Consider a randomized learner that makes prediction at each time $t$, according to a probability distribution:
%\[ y_t = \begin{cases} +1, &\text{w.p. } p_{t,+} \\ -1, &\text{w.p. } p_{t,-} \\ \bot, &\text{w.p. } 1-p_{t,+}-p_{t,-} \end{cases} \]
So given examples $(x_1,y_1), \ldots, (x_n, y_n)$, the cumulative mistake penalty upto time $n$ is as $ \sum_{t=1}^n I(y_t = -1)p_{t,+} + I(y_t = +1) p_{t,-}$ and the cumulative abstention penalty is $\sum_{t=1}^n (1 - p_{t,+} - p_{t,-})$.
%\E[\sum_{t=1}^m I(\hat{y}_t = -y_t)] =
%The cumulative abstention upto time $n$ is defined as
%\[  \]
%\E[\sum_{t=1}^m I(\hat{y}_t = \bot)]
We have the following result for finite hypothesis classes.
\begin{theorem}
There exists an instance domain $\calX$, a single-element hypothesis class $\calH$, such that the following holds. If $k < l$, then for any $a \geq 0$, there exists a strategy of the adversary satisfying $l$-bias assumption, such that any algorithm guaranteeing a cumulative mistake penalty at most $k$ in the randomized prediction model must have cumulative abstention penalty at least $a$.
\label{thm:finitelbrand}
\end{theorem}

For infinite hypothesis classes with Littlestone's dimension $d$, we have the following result.

\begin{theorem}
There exists an instance domain $\calX$ and a hypothesis class $\calH$ with Littlestone's dimension $d$ such that the following holds. If $k < l + d$, then for any $a \geq 0$, there exists a strategy of the adversary satisfying the $l$-bias assumption, such that any algorithm guaranteeing a mistake penalty of at most $k$ in the randomized prediction model must have cumulative abstention penalty at least $a$.
\label{thm:infinitelbrand}
\end{theorem}

\section{Related Work}

%The first model for learning with abstentions in the online setting is the KWIK model~\citep{LLWS11}. At each time $t$, the algorithm is presented with an example $x_t$ by the adversary. Then it is required to output either a correct prediction, or $\bot$ indicating abstention. ~\cite{LLWS11} proposes an algorithm which works in the finite hypothesis class setting in realizable case, for both classification and regression settings. ~\cite{SS11} considers the online regression problem in the same model, removing the assumption that the conditional expectation of $y_t$ given $x_t$ comes from some given hypothesis class. The closest to our work is~\cite{SZB10}, where the KWIK model is relaxed by allowing the learner to make $k$ mistakes. It shows that when working with a finite hypotheis class $\calH$, if $k$ mistakes are allowed, the number of abstentions is at most $(k+1)|\calH|^{\frac{1}{k+1}}$, by a reduction to the Egg Dropping Game~\citep{GF08, B04}. Our work extends~\cite{SZB10} in that we provide an optimal algorithm by exploiting structures in $\calH$, and that we allow the hypothesis class $\calH$ to be infinite. ~\cite{DZ13} provides efficient algorithms for mitake-abstention tradeoffs for the class of disjunctions. Another important line of work is conformal prediction~\citep{SV08}, which, given a conformity measure $R$ and an error probability measure $\delta$, shows a strategy for constructing confidence sets in an online manner that contain the correct label with probability $1-\delta$.

The first formal framework for online learning with abstentions is the Knows What It Knows (KWIK) model~\citep{LLWS11}, which works as follows. At time $t$, the learner is given an example $x_t$, and is expected to output either the correct label for $x_t$ or abstain from prediction. \cite{LLWS11} formalizes versions of this model for both classification and regression settings, and provides a classification algorithm that achieves no mistakes and $| \calH | - 1$ abstentions when the sequences provided are realizable for a finite hypothesis class $\calH$. \cite{SS11} provides algorithms for online regression that apply even when the realizability assumption is relaxed. 

Perhaps the most related to our work is~\cite{SZB10}, where the KWIK model is relaxed in the classification setting by allowing the learner to make $\leq k$ mistakes. This work presents an algorithm that, given a finite hypothesis $\calH$, can make at most $k$ mistakes with at most $(k+1)|\calH|^{\frac{1}{k+1}}$ abstentions. Our work extends~\cite{SZB10} in that we provide an optimal algorithm that exploits finer structures in $\calH$, and also in that we allow the hypothesis class $\calH$ to be infinite. ~\cite{DZ13} extends~\cite{SZB10} by providing efficient algorithms for the class of disjunctions. Finally, another important line of work for online classification with abstentions is conformal prediction~\citep{SV08}, which, given a conformity measure $R$ and an error probability measure $\delta$, shows a strategy for constructing confidence sets in an online manner that contain the correct label with probability $1-\delta$. Our framework differs from this line of work in that the conformity measure for us is not specified.

% by the adversary. Then it is required to output either a correct prediction, or $\bot$ indicating abstention. ~\cite{LLWS11} proposes an algorithm which works in the finite hypothesis class setting in realizable case, for both classification and regression settings. ~\cite{SS11} considers the online regression problem in the same model, removing the assumption that the conditional expectation of $y_t$ given $x_t$ comes from some given hypothesis class. 

%When a new example arrives, it predicts $\bot$ when there is disagreement on the current version space. It also discusses the regression setting, where a mistake is counted only if the prediction deviates from the true value by at least some threshold $\epsilon$.
%(which they call ``Agnostic KWIK" learning)

There is a large volume of literature on online classification when no abstentions are allowed.
The mistake bound model, initially proposed by~\citep{L87, A87}, considers online binary classification in the realizable case. \cite{L87} also introduces the standard optimal algorithm and optimal mistake bound (aka Littlestone's dimension $\Ldim(\calH)$). There has been much literature on developing algorithms for specific hypothesis classes in the mistake bound model; see~\cite{SB14, CL06} for examples.  \cite{BPS09} considers online classification (with no abstentions) in the agnostic case; they show that if the hypothesis class $\calH$ has finite Littlestone's dimension, then it is possible to design an online prediction algorithm that makes $l + \tilde{O}(\sqrt{\Ldim(\calH) T} + \Ldim(\calH))$ mistakes over $T$ rounds, where $l$ is the minimum error of any hypotheses in $\calH$. In follow-up work, ~\citep{RST10, RSS12, RST15a, RST15b} have developed a rich theory of online learning, and defined complexity measures such as sequential Rademacher complexity, and sequential covering number that characterize the complexity of online learning. However, this theory does not apply to online learning with abstentions.

%showing upper and lower bounds on the optimal mistake bounds via a non-constructive argument.

In the batch setting, the problem of classification with an abstention option has been both empirically and theoretically studied since the pioneering work of~\cite{C70}. It is however unclear how to directly apply the results in the batch setting to the online setting, because of the adversarial nature of the examples. ~\cite{HW06, BW08, YW10} consider classification where the decision to abstain is made based on thresholding a real-valued function that belongs to a fixed function class. ~\cite{FMS04} provides an algorithm that performs weighted majority style aggregation over a hypothesis class and abstains when the aggregate is close to zero. ~\cite{KKM12, KT14} study a related problem called reliable learning, and gives a predictor that achieves low error at the expense of abstentions. ~\cite{B16} considers the problem in transductive setting, where the goal is to make aggregated predictions with abstention based on an ensemble of classifiers, where some error upper bounds on individual classifiers are known. Finally, inspired by the active learning algorithm of~\citep{CAL94}, ~\cite{EYW10} proposes a abstention principle in the realizable case, which guarantees a zero error. ~\citet{EYW11} shows how to extend the idea to nonrealizable case, where the predictor has zero error with respect to the optimal hyothesis and~\cite{ZC14} gives an improved predictor when a nonzero amount of error is allowed.

\paragraph{Acknowledgements.} We thank NSF under IIS 1162581 for research support. CZ would like to thank Akshay Balsubramani and Haipeng Luo for helpful discussions.

\bibliographystyle{alpha}
\bibliography{agnostickwik}

\appendix

\section{Tree Shattering Coefficient and Sequential Growth Function}
\label{sec:tscsgf}
In this section, we show that the tree shattering coefficient $\calS(\calH, t)$ is at most the size of the sequential growth function(also known as maximal sequential zero covering number) of $\calH$~\citep{RST10}. We start with some notations.
\begin{definition}[Sequential Zero Cover and Sequential Zero Covering Number, see~\cite{RST10}]
A set $V$ of depth-$t$ trees is a sequential zero cover of $\calH$ on a depth-$t$ tree $\x$, if
\[ \forall h \in \calH, \forall \epsilon \in \cbr{\pm 1}^t, \exists \v \in V, s.t. \v_s(\epsilon) = h(\x_s(\epsilon)), s = 1,2,\ldots,t \]
The sequential zero covering number of a hypothesis class $\calH$ on a given tree $\x$ is defined as
\[ \calN(0, \calH, \x) := \min\cbr{|V|: \text{$V$ is a zero-cover of $\calH$ on $\x$}}\]
The maximal sequential zero covering number is the maximum sequential zero covering number of $\calH$ over all depth-$t$ $\calX$-valued trees $\x$, that is,
\[ \calN(0, \calH, t) := \max_{\x} \calN(0, \calH, \x)\]
\end{definition}

\begin{theorem}
For a given hypothesis class $\calH$ and integer $t \geq 0$,
\[ \calS(\calH, t) \leq \calN(0, \calH, t)\]
\end{theorem}
\begin{proof}
This is an immediate consequence of Lemma~\ref{lem:cardxv}.
\end{proof}

\begin{lemma}
Suppose we are given a $\calX$-valued tree $\x$ and a hypothesis class $\calH$. If $V$ is a sequential zero cover of $\calH$ on $\x$, then the size of $S(\calH, \x)$ is at most $|V|$.
\label{lem:cardxv}
%number of labellings achievable by $\calH$ over all depth-$t$ trees.
\end{lemma}
\begin{proof}
Recall that
\[ S(\calH, \x) = \cbr{ (\epsilon_1, \epsilon_2, \ldots, \epsilon_t) \in \cbr{\pm 1}^t: \epsilon_1 = h(\x_1(\epsilon)), \epsilon_2 = h(\x_2(\epsilon)), \ldots, \epsilon_t = h(\x_t(\epsilon)), h \in \calH } \]
Given an element $(\epsilon_1, \epsilon_2, \ldots, \epsilon_t)$ in $S(\calH, \x)$, there exists some $h$ in $\calH$ such that
\[ \epsilon_1 = h(\x_1(\epsilon)), \epsilon_2 = h(\x_2(\epsilon)), \ldots, \epsilon_t = h(\x_t(\epsilon)) \]
Since $V$ is a zero-cover of $\calH$, there exists a depth-$t$ tree $\v = (\v_1, \ldots, \v_t)$ in $V$ such that
\[ \v_1(\epsilon) = h(\x_1(\epsilon)), \v_2(\epsilon) = h(\x_2(\epsilon)), \ldots, \v_t(\epsilon) = h(\x_t(\epsilon))\]
Hence,
\[ \v_1(\epsilon) = \epsilon_1, \v_2(\epsilon) = \epsilon_2, \ldots, \v_t(\epsilon) = \epsilon_t \]
More explicitly,
\begin{equation}
\v_1 = \epsilon_1, \v_2(\epsilon_1) = \epsilon_2, \v_3(\epsilon_1, \epsilon_2) = \epsilon_3, \ldots, \v_t(\epsilon_1, \ldots, \epsilon_{t-1}) = \epsilon_t
\label{eqn:uniqpath}
\end{equation}
To summarize, for every $(\epsilon_1, \epsilon_2, \ldots, \epsilon_t)$ in $S(\calH, \x)$, there is a tree $\v$ in $V$ such that Equation~\eqref{eqn:uniqpath} holds. Since for each tree $\v$ there can be at most one $(\epsilon_1, \epsilon_2, \ldots, \epsilon_t)$ such that Equation~\eqref{eqn:uniqpath} holds, this implies that $|S(\calH, \x)| \leq |V|$.
\end{proof}

\section{Reducing the Expert Problem to Online Classification with Finite Class}
In this section, we show that the problem of Prediction with Expert Advice (abbrev. Expert Problem) with $l$-mistake assumption~\citep{CBFHW96, ALW06} can be cast to the problem studied in this paper, i.e. online classification with a finite hypothesis class with $l$-bias assumption. Specifically, in the expert problem, at each time $t$, the algorithm is given experts' advice $(x_{1,t}, \ldots, x_{N,t}) \in \cbr{-1,1}^N$, and predicts $\hat{y}_t \in \cbr{-1,1,\bot}$. Then adversary reveals label $y_t \in \cbr{-1,1}$. The $l$-mistake assumption states that there is an expert $i$ that makes at most $l$ mistakes throughout the process, i.e.
\[ \exists i, |\cbr{t: x_{i,t} \neq y_t}| \leq l \]
For $i = 1,2,\ldots,N$, define hypothesis $h_i: \R^{N+1} \to \R$ as mapping a $(N+1)$-dimensional vector to its $i$th coordinate. Define hypothesis class $\calH_N := \cbr{h_i: i = 1,\ldots,N}$. We have the following result relating the $l$-mistake assumption to $l$-bias assumption; the intuition is to concatenate a new coordinate at the end of the experts' advice to make all the examples shown distinct.

\begin{proposition}
 The following are equivalent:
\begin{enumerate}[(a)]
\item The sequence of expert advice and labels $(x_{1,t}, \ldots, x_{N,t}), y_t, t=1,2,\ldots$ satisfies $l$-mistake assumption.
\item The sequence $x_t = (x_{1,t}, \ldots, x_{N,t}, t), y_t, t=1,2,\ldots$ satisfies $l$-bias assumption with respect to $\calH_N$.
\end{enumerate}
\end{proposition}

\begin{proof}
We show the implication in both directions.
\begin{enumerate}
\item[($\Rightarrow$)] If $(x_{1,t}, \ldots, x_{N,t}), y_t, t=1,2,\ldots$ satisfies $l$-mistake assumption, then there is $i \in \cbr{1,\ldots,N}$ such that
\[ M_i = |\cbr{t: x_{i,t} \neq y_t}| \leq l \]
Hence, $h_i$ is correct on all but the rounds $t$ in $M_i$, i.e. on examples $\cbr{(x_{1,t}, \ldots, x_{N,t}, t): t \in M_i}$, which are distinct and has size at most $l$. Therefore, the sequence $\cbr{(x_{1,t}, \ldots, x_{N,t}, t)}, t =1,2,..$ satisfies $l$-bias assumption with respect to $\calH_N$.

\item[($\Leftarrow$)] If the sequence $(x_{1,t}, \ldots, x_{N,t}, t), y_t, t=1,2,\ldots$ satisfies $l$-bias assumption with respect to $\calH_N$, then there exists $h_i$ that is correct on all but $p \leq l$ examples shown. That is, $p$, the size of the set
\[ M_i = |\cbr{t: x_{i,t} \neq y_t}|\]
is at most $l$. This immediately implies (a). \qedhere
\end{enumerate}
\end{proof}

An immediate consequence of the above proposition is that, for an instance of the expert problem with $l$-mistake assumption, we can convert it to an instance of online classification in $\calH_N$ under $l$-bias assumption, and apply SOA.DK on $\calH_N$ to get mistake-abstention tradeoffs.

\section{A Note on the Recursive Definition of $\ELdim$}
At the end of Section~\ref{sec:realizable}, we give a recursive definition on
$\ELdim(\calH, k)$ when $k \geq 1$:
\begin{eqnarray}
  \ELdim(\calH, k) &:=& \max_x \max_{y \in \{-1,+1\}} \min \big( \ELdim(\calH[(x,y)], k), \ELdim(\calH[(x,-y)], k-1) \big) \nonumber \\
                   &=&  \max_x \max \Big( \min \big( \ELdim(\calH[(x,-1)], k), \ELdim(\calH[(x,+1)], k-1) \big), \nonumber \\
                        && \quad\qquad\qquad \min \big( \ELdim(\calH[(x,+1)], k), \ELdim(\calH[(x,-1)], k-1) \big) \Big) \label{eqn:treedef}
\end{eqnarray}
On the other hand, by the definition of Algorithm~\ref{alg:esoa}, we also have
\begin{eqnarray}
\ELdim(\calH, k) &=&  \max_x \min \Big( \ELdim(\calH[(x,-1)], k-1), \ELdim(\calH[(x,+1)], k-1) \big), \nonumber \\
&& \quad\qquad\qquad \max \big( \ELdim(\calH[(x,+1)], k), \ELdim(\calH[(x,-1)], k) \big) \Big) \label{eqn:algdef}
\end{eqnarray}

In this section we show that these two definition are indeed equivalent. First we need a
simple observation.
\begin{lemma}
If $A,B,C$ are real numbers, then $\min(\max(A,B),C) = \max(\min(A,C), \min(B,C))$.
\label{lem:distrib}
\end{lemma}
\begin{lemma}
The right hand sides of Equations~\eqref{eqn:treedef} and~\eqref{eqn:algdef} are equal.
\end{lemma}
\begin{proof}
Fix $x$ in $\calX$. We have:
\small
\begin{eqnarray*}
  && \min \Big( \ELdim(\calH[(x,-1)], k-1), \ELdim(\calH[(x,+1)], k-1) \big),\max \big( \ELdim(\calH[(x,+1)], k), \ELdim(\calH[(x,-1)], k) \big) \Big) \\
  &=& \min \Bigg( \ELdim(\calH[(x,-1)], k-1), \min \Big( \ELdim(\calH[(x,+1)], k-1) \big),\max \big( \ELdim(\calH[(x,+1)], k), \ELdim(\calH[(x,-1)], k) \big) \Big) \Bigg) \\
  &=& \min \Bigg( \ELdim(\calH[(x,-1)], k-1), \max \Big( \ELdim(\calH[(x,+1)], k), \min\big( \ELdim(\calH[(x,-1)], k), \ELdim(\calH[(x,+1)], k-1) \big) \Big) \Bigg) \\
  &=& \max \Bigg( \min \Big( \ELdim(\calH[(x,-1)], k-1), \ELdim(\calH[(x,+1)], k) \Big), \min \Big( \ELdim(\calH[(x,-1)], k), \ELdim(\calH[(x,+1)], k-1) \Big) \Bigg)
\end{eqnarray*}
\normalsize
where the first equality is from the associativity of $\min$;
the second equality
is from Lemma~\ref{lem:distrib} and
$\ELdim(\calH[(x,+1)], k-1) \geq \ELdim(\calH[(x,+1)], k)$;
the third equality
is from Lemma~\ref{lem:distrib} and
$\ELdim(\calH[(x,-1)], k-1) \geq \ELdim(\calH[(x,-1)], k)$.
Taking the maximum over $x \in \calX$ proves the lemma.
\end{proof}
\section{Proofs from Section~\ref{sec:realizable}}
\label{sec:pfrealizable}

We first provide some auxiliary lemmas regarding properties of extended mistake trees and extended Littlestone's dimension. This will serve as the basis of the proof of Lemma~\ref{lem:perfguar}.

We state a property about subtrees of a $(k,m)$-difficult extended mistake tree.
\begin{lemma}[Recursive Property of Extended Mistake Trees]
Suppose we are given hypothesis class $\calH$ that has an extended mistake tree $T$ with root $x$, left subtree $T_{-1}$, right subtree $T_{+1}$ and integers $k \geq 0, m \geq 1$. For the root node $x$, denote by $e_l$ its downward left solid edge, $e_r$ its downward right solid edge, and $e_d$ its downward dashed edge. Denote by $y$ the label of $e_d$.
\begin{enumerate}
\item[(i)] The following statements are equivalent: (a) $T$ is $(0,m)$-difficult. (b) $T_y$ is $(0,m-1)$-difficult.
\item[(ii)] For $k \geq 1$, the following statements are equivalent: (a) $T$ is $(k,m)$-difficult. (b) $T_{-y}$ is $(k-1,m-1)$-difficult, and $T_y$ is $(k,m-1)$-difficult.

\end{enumerate}
\label{lem:recursive}
\end{lemma}

\begin{proof}[Proof of Lemma~\ref{lem:recursive}]

Without loss of generality, suppose $y = +1$. The case of $y = -1$ can be shown symmetrically.

\noindent\textbf{Proof of item (i):}
We show the implication in both directions.
\begin{enumerate}
\item[($\Rightarrow$)] Consider a root to leaf path $p$ in $T_{+1}$ that uses no solid edges.
 Now consider path $p_+$, the result of prepending the root node $x$ and the
 downward dashed edge from root $x$ to its right child onto $p$, i.e. $p_+ = xe_dp$.
 It can be seen that $p_+$ uses no solid edges, and $l(p_+) = l(p) + 1$.
 Since $T$ is $(0,m)$-difficult, $l(p_+) \geq m$, therefore $l(p) \geq m-1$,
 thus showing $T_{+1}$ is $(0,m-1)$-difficult.
\item[($\Leftarrow$)] Consider a root to leaf path $p = v_1e_1v_2e_2\ldots v_ne_nv_{n+1}$
in $T$ that uses no solid edges. The first edge of $p$ must be the downward dashed edge $e_d$.
Define path $p_-$ as the result of deleting the first node $v_1 = x$ and the first edge $e_1$ from $p$,
i.e. $p_- = v_2e_2\ldots v_ne_nv_{n+1}$.
Since $T_{-1}$ is $(0,m-1)$-difficult, we get that $l(p_-) \geq m - 1$. Therefore $l(p) = l(p_-) + 1 \geq m$.
Therefore, any path $p$ in $T$ that uses no solid edges must be of length at least $m$.
Thus, $T$ is $(0,m)$-difficult.
\end{enumerate}

\noindent\textbf{Proof of item (ii):}
We show the implication in both directions.
\begin{enumerate}
\item[($\Rightarrow$)]
  \begin{enumerate}[(1)]
    \item Consider a root to leaf path $p$ in $T_{-1}$ that uses at most $k-1$ solid edges.
    Now consider path $p_+$, the result of prepending the root node $x$
    and the downward edge from root $x$ to its left child onto $p$,
    i.e. $p_+ = xe_lp$.
    It can be seen that $p_+$ uses at most $k$ solid edges, and $l(p_+) = l(p) + 1$.
    Since $T$ is $(k,m)$-difficult, $l(p_+) \geq m$, therefore $l(p) \geq m-1$,
    thus showing $T_{-1}$ is $(k-1,m-1)$-difficult.

    \item Consider a root to leaf path $p$ in $T_{+1}$ that uses at most $k$ solid edges.
    Now consider path $p_+$, the result of prepending the root node $x$ and the downward
    dashed edge from root $x$ to its right child onto $p$, i.e. $p_+ = xe_dp$.
    It can be seen that $p_+$ uses at most $k$ solid edges,
    and $l(p_+) = l(p) + 1$. Since $T$ is $(k,m)$-difficult, $l(p_+) \geq m$,
    therefore $l(p) \geq m-1$, thus showing $T_{+1}$ is $(k,m-1)$-difficult.
  \end{enumerate}

 \item[($\Leftarrow$)]
Consider a root to leaf path $p = v_1e_1v_2e_2\ldots v_ne_nv_{n+1}$ in $T$ that uses at most $k$ solid edges. Define path $p_-$ as the result of deleting the first node $v_1 = x$ the first edge $e_1$ from $p$, i.e. $p_- = v_2e_2\ldots v_ne_nv_{n+1}$.
\begin{enumerate}[(1)]
  \item If the first edge of $p$ is a downward edge from root $x$ to its left child,
  then $p_-$ is a root to leaf path in $T_{-1}$, and uses at most $k-1$ solid edges,
  since the first edge $e_1$ has to be a solid edge. Since $T_{+1}$ is $(k-1,m-1)$-difficult,
   we get that $l(p_-) \geq m - 1$.

  \item If the first edge of $p$ is the downward edge from root $x$ to its right child,
  then $p_-$ is a root to leaf path in $T_{+1}$, and uses at most $k$ solid edges.
  Since $T_{-1}$ is $(k,m-1)$-difficult, we get that $l(p_-) \geq m - 1$.
\end{enumerate}
In both cases, $l(p_-) \geq m - 1$. Hence $l(p) = l(p_-) + 1 \geq m$.
In summary, any path $p$ in $T$ that uses at most $k$ solid edges must be of length $m$. Thus, $T$ is $(k,m)$-difficult.
\end{enumerate}
\end{proof}

Built upon Lemma~\ref{lem:recursive}, we obtain the following result regarding $\calH$'s extended Littlestone's dimension.
\begin{lemma}[Recursive Property of $\ELdim$]
Suppose we are given hypothesis class $\calH$ and integers $k \geq 0,m \geq 1$.
\begin{enumerate}
\item[(i)] The following statements are equivalent: (a) $\ELdim(\calH, 0)$ is at least $m$. (b) There exists $(x,y)$ such that $x$ is in $\DIS(\calH)$, and $\ELdim(\calH[(x,y)], 0)$ is at least $m-1$.

\item[(ii)] If $k \geq 1$, the following statements are equivalent: (a) $\ELdim(\calH, k)$ is at least $m$. (b) There exists $(x,y)$ such that both $\ELdim(\calH[(x,y)], k)$ and $\ELdim(\calH[(x,-y)], k-1)$ are at least $m - 1$.
\end{enumerate}
\label{lem:eldimrecursive}
\end{lemma}

\begin{proof}[Proof of Lemma~\ref{lem:eldimrecursive}]~\\
\noindent\textbf{Proof of item (i):}
We show the implication in both directions.
\begin{enumerate}
\item[($\Rightarrow$)]
Suppose $\ELdim(\calH, 0) \geq m$. Then $\calH$ has a $(0,m)$-difficult mistake tree $T$.
Let $x$ be the root of $T$, and $y \in \cbr{-1,+1}$ be the label of the root's
downward dashed edge. Since $T$ is a full binary tree, there must be leaves
in both the left subtree and the right subtree of the root, i.e. there exist
$h_1, h_2$ in $\calH$, $h_1(x) = -1$ and $h_2(x) = -1$. Thus, $x$ is in $\DIS(\calH)$.
By Lemma~\ref{lem:recursive}, $T_y$ is a $(0,m-1)$-difficult extended mistake tree
with respect to $\calH[(x,y)]$. The result follows.

\item[($\Leftarrow$)]
Suppose there exists an example $(x,y)$ such that $x \in \DIS(\calH)$
and $\ELdim(\calH[(x,y)], 0) \geq m-1$. Then, $\calH[(x,y)]$ has a $(0,m-1)$-difficult
 extended mistake tree $T_y$ and $\calH[(x,-y)]$ has a zeroth order mistake tree $T_{-y}$.
 Construct a new tree $T$, where its root is $x$, and its subtrees are $T_y$ and $T_{-y}$ respectively.
 The dashed downward edge is connected to the subtree $T_y$. By Lemma~\ref{lem:recursive},
 $T$ is a $(0,m)$-difficult extended mistake tree with respect to $\calH$. The result follows.
\end{enumerate}

\noindent\textbf{Proof of item (ii):}
We show the implication in both directions.
\begin{enumerate}
\item[($\Rightarrow$)]
Suppose $\ELdim(\calH, k) \geq m$. Then $\calH$ has a $(k,m)$-difficult mistake tree $T$.
Let $x$ be the root of $T$, and $y \in \cbr{-1,+1}$ be the root's downward dashed edge label.
By Lemma~\ref{lem:recursive}, $T_{-y}$ is a $(k-1,m-1)$-difficult extended mistake tree with
respect to $\calH[(x,-y)]$, and $T_y$ is a $(k,m-1)$-difficult extended mistake tree with
respect to $\calH[(x,y)]$. The result follows.

\item[($\Leftarrow$)]
Suppose there exists an example $(x,y)$ such that both $\ELdim(\calH[(x,y)], k) \geq m-1$
and $\ELdim(\calH[(x,-y)], k-1) \geq m-1$. Then, $\calH[(x,y)]$ has a $(k,m-1)$-difficult
extended mistake tree $T_y$ and $\calH[(x,-y)]$ has a $(k-1,m-1)$-difficult extended mistake tree $T_{-y}$.
Now construct a new tree $T$, where its root is $x$, and its subtrees are $T_y$ and $T_{-y}$ respectively.
The dashed downward edge is connected to the subtree $T_y$. By Lemma~\ref{lem:recursive},
$T$ is a $(k,m)$-difficult extended mistake tree with respect to $\calH$. The result follows.
\end{enumerate}

\end{proof}

\begin{proof}[Proof of Lemma~\ref{lem:emtusage}]
Since $\ELdim(\calH, k) \geq m$, there is a $(k,m)$-difficult extended mistake tree $T_{\calH}$ with respect to $\calH$. We consider the the strategy of the adversary associated with $T_{\calH}$.
Now consider any deterministic learning algorithm $\calA$ that guarantees at most $k$ mistakes.
Since $\calA$ is deterministic, the interaction between $\calA$ and the adversary follows some path $p$ from root to leaf.
The number of mistakes is equal to the number of solid edges in $p$, and the number of abstentions is equal to the number of dashed edges in the $p$.
Since $\calA$ guarantees $k$ mistakes, $p$ must contain at most $k$ solid edges, thus it must be of length at least $m$, as $T_{\calH}$ is $(k,m)$-difficult.
Therefore, the number of nontrivial rounds of $\calA$ is at least $m$.
\end{proof}

\begin{proof}[Proof of Lemma~\ref{lem:perfguar}]
We prove the lemma by joint induction on $(k,m)$.
\paragraph{Base Case.} Consider pairs $(k, m)$, where $k = 0$ or $m = 0$.
\begin{enumerate}[(1)]
\item For $m = 0$ and $k \geq 0$, if there is no $(k,1)$-difficult extended mistake tree, then for all $x \in \calX$, $V$ predicts unanimously on $x$.
Otherwise, there are two hypotheses $h_1$ and $h_2$ and an example $x$ such that $h_1(x) = -1$ and $h_2(x) = +1$. Consider extended mistake tree $T$ as follows. $T$ has $x$ as its root, and $h_1$ and $h_2$ are leaves directly connecting to the root, where $h_1$ is on the left and $h_2$ is on the right. The downward dashed edge is connected to the right, i.e. has label $+1$. It can be seen that $T$ is $(k,1)$-difficult for any $k \geq 0$. Therefore, Algorithm~\ref{alg:esoa} always predicts correctly, and there will be no nontrivial rounds subsequently.

\item For $k = 0$ and $m \geq 0$, we show the result by induction on $m$. The base case $m = 0$ has been shown in (1). For the inductive case, assume the inductive hypothesis holds for $m' \leq m-1$.
Now, given a hypothesis class $V$ such that $\ELdim(V, 0)$ is at most $m$. Consider the first nontrivial round $t$ when running Algorithm~\ref{alg:esoa} with version space $V$. The example $x_t$ must be in $\DIS(V)$, and the algorithm outputs $\hat{y}_t = \bot$. We claim that the resulting version space $V[(x_t, y_t)]$ is such that $\ELdim(V[(x_t,y_t)],0) \leq m-1$. Indeed, suppose $\ELdim(V[(x_t,y_t)],0) \geq m$, then by Lemma~\ref{lem:eldimrecursive}, $\ELdim(V,0) \geq m+1$, which is a contradiction.

Note that from time $t+1$ on, the adversary is only allowed to show $V[(x_t, y_t)]$-realizable sequences. By inductive hypothesis, Algorithm~\ref{alg:esoa} runs on $V[(x_t,y_t)]$ and achieves $(0,m-1)$-SZB bound from time $t+1$ on. Therefore, Algorithm~\ref{alg:esoa} achieves $(0,m)$-SZB bound throughout the process.
%(Otherwise a $(0,d+1)$-difficult extended mistake tree can be built.)
%Therefore, there exists a $(0,d-1)$-difficult extended mistake tree $T_{\calH[(x,y)]}$ with respect to $\calH[(x,y)]$. Since $x \in \DIS(\calH)$, we can find a classifier $h_{-y} \in \calH$ consistent with $(x,-y)$. We construct an extended mistake tree $T_{\calH}$ by setting node $x$ as root and let its left and right branches to be leaf $h_{-y}$ and $T_{\calH[(x,y)]}$, resepctively. It is easy to see that $T_{\calH}$ is a $(0,d)$-difficult mistake tree. This completes the argument for general $k = 0$.
\end{enumerate}

\paragraph{Inductive Case.} Consider pairs $(k,m)$ where $k \geq 1$ and $m \geq 1$. Assume for all $k' \leq k$, $m' \leq m$ and $k'+m' \leq k+m-1$, the inductive hypothesis holds. Now, consider a hypothesis class $V$ such that $\ELdim(V,k) \leq m$.
%This implies that for all $x \in \DIS(V)$, both of the following hold:
%1. $V[(x,+1)]$ does not have a $(k-1,d-1)$-extended mistake tree, or $V[(x,-1)]$ does not have a $(k,d-1)$-extended mistake tree.
%2. $V[(x,-1)]$ does not have a $(k-1,d-1)$-extended mistake tree, or $V[(x,+1)]$ does not have a $(k,d-1)$-extended mistake tree.
Consider the first nontrivial round $t$ when we run Algorithm~\ref{alg:esoa} on $V$. The example $x_t$ must be in $\DIS(V)$. According to Algorithm~\ref{alg:esoa}'s prediction $\hat{y}_t$, we consider three cases separately,

\paragraph{Case 1: $\hat{y}_t = -1$.} In this case, since round $t$ is nontrivial, $y_t = -\hat{y}_t = +1$. We claim that $\ELdim(V[(x_t,+1)],k-1) \leq m-1$. Indeed, assume (for the sake of contradiction) that $m_{-1} \geq \ELdim(V[(x_t,+1)],k-1) \geq m$. By definition of Algorithm~\ref{alg:esoa}, $\ELdim(V[(x_t,-1)],k-1) = m_{+1} \geq m_{-1} \geq m$. Hence, for any $y \in \cbr{-1, +1}$, $\ELdim(V[(x_t,y)],k-1) \geq m$.

Also by definition of Algorithm~\ref{alg:esoa}, $\max(\ELdim(V[(x_t,+1)],k), \ELdim(V[(x_t,-1)],k)) = m_\bot \geq m$.
Thus, there exists some $\hat{y} \in \cbr{-1, +1}$ such that
\[ \ELdim(V[(x_t,\hat{y})],k) \geq m \]
Therefore, for $\hat{y}$, we have $\ELdim(V[(x_t,\hat{y})],k) \geq m$ and $\ELdim(V[(x_t,-\hat{y})],k-1) \geq m$. By Lemma~\ref{lem:eldimrecursive}, $\ELdim(V, k) \geq m+1$, which is a contradiction.

%In this case, since round $t$ is nontrivial, $y_t = -\hat{y}_t = +1$. We claim $V[(x_t,+1)]$ does not have a $(k-1,d)$-difficult extended mistake tree. To see why, assume $V[(x_t,+1)]$ has a $(k-1,d)$-difficult extended mistake tree $T_+$. Then $m_{-1} = \ELdim(V[(x_t,+1)], k-1) \geq d$. Since $\hat{y}_t = -1$, we have $\ELdim(V[(x_t,-1)], k-1) = m_{+1} \geq m_{-1} \geq d$, implying $V[(x_t,-1)]$ has a $(k-1,d)$-difficult extended mistake tree $T_-$. Now construct an extended mistake tree $T$ with respect to $\calH$, by defining the root as $x_t$, the left subtree as $T_-$ and the right subtree as $T_+$. The dashed downward edge connects to the left subtree(say). It can be seen that $T$ is $(k,d+1)$-difficult, contradiction.

Note that from time $t+1$ on, the adversary is only allowed to show $V[(x_t, y_t)]$-realizable sequences. By inductive hypothesis, Algorithm~\ref{alg:esoa} runs on $V[(x_t,y_t)]$ with mistake budget $k-1$ and achieves $(k-1,m-1)$-SZB bound from round $t+1$ on. Therefore, Algorithm~\ref{alg:esoa} achieves $(k,m)$-SZB bound throughout the process.

\paragraph{Case 2: $\hat{y}_t = +1$.} This case is symmetric to Case 1.

\paragraph{Case 3: $\hat{y}_t = \bot$.}
We first claim that $\ELdim(V[(x_t,-1)], k) \leq m-1$. Indeed, assume  (for the sake of contradiction) that $\ELdim(V[(x_t,-1)],k) \geq m$. By definition of Algorithm~\ref{alg:esoa}, $m_{-1} \geq m_\bot \geq m$, that is
\[ \ELdim(V[(x_t,+1)],k-1) \geq m\]
By Lemma~\ref{lem:eldimrecursive}, $\ELdim(V, k) \geq m+1$, contradiction. Symmetrically, one also has $\ELdim(V[(x_t,+1)], k) \leq m-1$.

%We first claim that $V[(x_t,-1)]$ does not have a $(k,d)$-difficult extended mistake tree. To see why, assume $V[(x_t,-1)]$ does has a $(k,d)$-difficult extended mistake tree $T_-$. Then
%\[ d_\bot = \max(\ELdim(V[(x_t,-1)], k), \ELdim(V[(x_t,+1)], k)) \geq d\]
%Since $\hat{y}_t = \bot$, we have $\ELdim(V[(x_t,+1)], k-1) = m_{-1} \geq d_\bot \geq d$, implying $V[(x_t,+1)]$ has a $(k-1,d)$-difficult extended mistake tree $T_+$. Now construct an extended mistake tree $T$ with respect to $\calH$, by defining the root as $x_t$, the left subtree as $T_-$ and the right subtree as $T_+$. The dashed downward edge connects to the right subtree. It can be seen that $T$ is $(k,d+1)$-difficult, contradiction.

%Symmetrically, $V[(x_t,+1)]$ does not have a $(k,d)$-difficult extended mistake tree.

Hence, irrespective of the outcome $y_t \in \cbr{-1,+1}$, the resulting version space $V[(x_t,y_t)]$ satisfies that $\ELdim(V[(x_t,y_t)], k) \leq m-1$. Note that from time $t+1$ on, the adversary is only allowed to show $V[(x_t, y_t)]$-realizable sequences. By inductive hypothesis, Algorithm~\ref{alg:esoa} runs on $V[(x_t,y_t)]$ with mistake budget $k$, and achieves $(k,m-1)$-SZB bound from round $t+1$ on. Therefore, Algorithm~\ref{alg:esoa} achieves $(k,m)$-SZB bound throughout the process.

In summary, Algorithm~\ref{alg:esoa}, when run on $V$, achieves $(k,m)$-SZB bound. This completes the induction.
%all hypothesis class $V$ with $\ELdim(V, k') \geq d'$ has a $(k',d')$-difficult mistake tree. Consider hypothesis class $\calH$ such that $\ELdim(\calH,k) \geq d$.
\end{proof}

\begin{proof}[Proof of Theorem~\ref{thm:optimality}]

\begin{enumerate}[(a)]
\item This is an immediate consequence of Lemma~\ref{lem:perfguar}.
%We will show that if the adversary follows the strategy associated with the extended mistake tree $T_{\calH}$, then any algorithm predicting at most $k$ mistakes must have at least $d$ nontrivial rounds.  \cz{Can formalize the argument by induction if this is too vague.}
\item By Lemma~\ref{lem:emtusage}, there is a strategy of the adversary such that any deterministic learner guaranteeing at most $k$ mistakes must have at least $m$ nontrivial rounds. Therefore, no deterministic learner can achieve a $(k,m-1)$-SZB bound.
\end{enumerate}
\end{proof}

\begin{proof}[Proof of Theorem~\ref{thm:eldimldim}]

Recall that $\Ldim(\calH) = d < \infty$. We show the equality by showing inequalities
in both sides.
\begin{enumerate}[(1)]
\item We first show $\ELdim(\calH, d) \leq d$.
Indeed, SOA is guaranteed to make at most $d$ mistakes and no abstentions
for $\calH$-realizable sequences. This has a total of at most $d$ nontrivial rounds.
Now, by Lemma~\ref{lem:emtusage}, if $\ELdim(\calH, d) \geq d + 1$,
SOA must have at least $d+1$ nontrivial rounds, contradiction.
\item On the other hand, since $\Ldim(\calH) = d$,
 there is a depth-$d$ mistake tree $T$ with respect to $\calH$.
 Consider the following modification of $T$: for each internal node,
 add a dashed downward edge to its right child. It can be seen that the resulting tree,
  $\tilde{T}$, is a $(d,d)$-difficult extended mistake tree.
  Therefore $\ELdim(\calH, d) \geq d$.
\end{enumerate}
In summary, $\ELdim(\calH, d) = d$.
\end{proof}

\section{Proofs from Section~\ref{sec:eldimension}}
\label{sec:pfeldimension}

\begin{proof}[Proof of Lemma~\ref{lem:finitegrowth}]
For any depth-$t$ tree $\x$, note that
\[ |S(\calH, \x)| \leq |\calH|\]
Therefore,
\[ \calS(\calH, t) = \max_{\x} |S(\calH, \x)| \leq |\calH|. \qedhere \]
\end{proof}

%We list some basic properties of tree shattering coefficient below.

\begin{lemma}[Recursive Formula]
For a hypothesis class $\calH$ and $t \geq 1$, we have
\[ \calS(\calH,t) = \max_{x \in \calX} (\calS(\calH[(x, -1)],t-1) + \calS(\calH[(x, +1)],t-1)) \]
\label{lem:recursivegrowth}
\end{lemma}

We need the following notation of subtrees to give the proof of Lemma~\ref{lem:recursivegrowth}.
\begin{definition}[Subtrees, see~\cite{RST10}]
Given a depth-$t$ tree $\x$, the left subtree $\x^l$ of $\x$ at the root is defined as $t-1$ mappings $(\x_1^l, \ldots, \x_{t-1}^l)$, where $\x_i^l(\epsilon) = \x(\cbr{-1} \times \epsilon)$, for $\epsilon \in \cbr{\pm 1}^{t-1}$. The right subtree $\x^r$ of $\x$ at the root is defined as $t-1$ mappings $(\x_1^r, \ldots, \x_{t-1}^r)$, where $\x_i^r(\epsilon) = \x(\cbr{+1} \times \epsilon)$, for $\epsilon \in \cbr{\pm 1}^{t-1}$.
\end{definition}

\begin{proof}[Proof of Lemma~\ref{lem:recursivegrowth}]
Consider the definition of $\calS(\calH,t)$:
\[ \max_\x |\cbr{ (\epsilon_1, \epsilon_2, \ldots, \epsilon_t) \in \cbr{\pm 1}^t: \epsilon_1 = h(\x_1(\epsilon)), \epsilon_2 = h(\x_2(\epsilon)), \ldots, \epsilon_t = h(\x_t(\epsilon)), h \in \calH }| \]
This can be alternatively written as
\begin{eqnarray*}
\max_{\x} |\cbr{ (-1, \sigma_1, \ldots, \sigma_{t-1}) \in \cbr{\pm 1}^{t-1}: \sigma_1 = h(\x_1^l(\sigma)), \ldots, \sigma_t = h(\x_{t-1}^l(\sigma)), h \in \calH[(\x_1,-1)] } \\
\cup \cbr{ (+1, \sigma_1, \ldots, \sigma_{t-1}) \in \cbr{\pm 1}^{t-1}: \sigma_2 = h(\x_1^r(\sigma)), \ldots, \sigma_t = h(\x_{t-1}^r(\sigma)), h \in \calH[(\x_1,+1)] }|
\end{eqnarray*}
The above is equal to
\begin{eqnarray*}
\max_{\x_1 \in \calX} \{ \max_{\x^l} |\cbr{ (-1, \sigma_1, \ldots, \sigma_{t-1}) \in \cbr{\pm 1}^{t-1}: \sigma_1 = h(\x_1^l(\sigma)), \ldots, \sigma_{t-1} = h(\x_{t-1}^l(\sigma)), h \in \calH[(\x_1,-1)] } \\
+ \max_{\x^r} \cbr{ (+1, \sigma_1, \ldots, \sigma_{t-1}) \in \cbr{\pm 1}^{t-1}: \sigma_1 = h(\x_1^l(\sigma)), \ldots, \sigma_{t-1} = h(\x_{t-1}^l(\sigma)), h \in \calH[(\x_1,+1)] }| \}
\end{eqnarray*}
Note that the right hand side is precisely $\max_{x_1 \in \calX} (\calS(\calH[(x_1, -1)],t-1) + \calS(\calH[(x_1, +1)],t-1))$. The lemma follows.
\end{proof}

%\begin{proof}
%Fix a tree, represented by $\cbr{x_s: \cbr{\pm 1}^{s-1} \to \calX}_{s=1}^t$. Note that
%\[ |\cbr{ (\epsilon_1, \epsilon_2, \ldots, \epsilon_t) \in \cbr{\pm 1}^t: \epsilon_1 = h(x_1), \epsilon_2 = h(x_2(\epsilon_1)), \ldots, \epsilon_t = h(x_t(\epsilon_1, \ldots, \epsilon_{t-1})), h \in \calH }|\]
%is of cardinality at most $|\calH|$.
%\end{proof}
Now we are ready to prove Lemma~\ref{lem:ldimgrowth}.

\begin{proof}[Proof of Lemma~\ref{lem:ldimgrowth}]
We prove the result by joint induction on $(t,d)$.
\paragraph{Base Case:} Consider $t = 0$ or $d = 0$. If $t = 0$, then $\calS(\calH,0) \leq 1 = \binom{0}{\leq d}$. If $d = 0$, then $\calS(\calH,t) \leq 1 = \binom{t}{\leq 0}$.

\paragraph{Inductive Case:} For $t \geq 1$ and $d \geq 1$, assume the result holds for $(t',d')$ such that $t' \leq t$, $d' \leq d$ and $t'+d' \leq t+d-1$.
First by Lemma~\ref{lem:recursivegrowth}, for some $x$ in $\calX$, $\calS(\calH, t) \leq \calS(\calH[(x,-1)],t-1) + \calS(\calH[(x,+1)],t-1)$.

Second, Since $\Ldim(\calH) = d$, for $x$, there exists $y \in \cbr{-1, +1}$ such that $\Ldim(\calH[(x, y)]) \leq d-1$ and $\Ldim(\calH[(x,-y)])) \leq d$. Hence by inductive hypothesis, there exists $y \in \cbr{-1,+1}$ such that $\calS(\calH[(x,y)],t-1) \leq \binom{t-1}{\leq d-1}$ and $\calS(\calH[(x,-y)],t-1) \leq \binom{t-1}{\leq d}$. Therefore
\[ \calS(\calH,t) \leq \calS(\calH[(x,-1)],t-1) + \calS(\calH[(x,+1)],t-1) \leq \binom{t-1}{\leq d-1} + \binom{t-1}{\leq d} \leq \binom{t}{\leq d}\]
This completes the induction.
\end{proof}

\begin{proof}[Proof of Theorem~\ref{thm:growth}]
For any integer $m$, if $m \leq \ELdim(\calH, k)$, then by Lemma~\ref{lem:growth},
\[ \calS(\calH,m) \geq \binom{m}{\leq k+1} \]
This implies that
\[ m \leq \sup\left\{t: \binom{t}{\leq k+1} \leq \calS(\calH,t)\right\} \]
Taking $m = \ELdim(\calH, k)$, we get the theorem.
\end{proof}

\begin{lemma}
Suppose $k, t$ are nonnegative integers. If $\ELdim(\calH,k) \geq t$, then $\calS(\calH,t) \geq \binom{t}{\leq k+1}$.
\label{lem:growth}
\end{lemma}
\begin{proof}[Proof of Lemma~\ref{lem:growth}]
By joint induction on $(k,t)$.
\paragraph{Base Case:} We consider $(k,t)$ pairs where $k = 0$ or $t = 0$.
\begin{enumerate}[(1)]
\item For $t=0$, $\ELdim(\calH,k) \geq 0$ implies that $\calH$ is nonempty. Thus, $\calS(\calH,0) = 1 \geq \binom{0}{\leq k+1}$.
\item For $k=0$, we prove the result by induction on $t$. The case of $t=0$ has been shown in (1). For the inductive case, by Lemma~\ref{lem:eldimrecursive}, there exists $(x,y)$ such that $x \in \DIS(\calH)$ and $\ELdim(\calH[(x,y)], 0) \geq t - 1$. Thus, by inductive hypothesis, $\calS(\calH[(x,y)], t - 1) \geq t$. Also, since $\calH[(x,-y)]$ is nonempty, we get $\calS(\calH[(x,-y)], t - 1) \geq 1$. Thus,
\[ \calS(\calH, t) \geq \calS(\calH[(x,y)], t - 1) + \calS(\calH[(x,-y)], t - 1) \geq t + 1\]
This completes the proof for $k = 0$.
\end{enumerate}
\paragraph{Inductive Case:} For $t \geq 1$ and $k \geq 1$, suppose the inductive hypothesis holds for any $(k',t')$ such that $k' \leq k$, $t' \leq t$, $k' + t' \leq k + t + 1$.

%The inductive hypothesis states that, for any $k'$, $t'$ such that $k' \leq k$, $t' \leq t$, $k' + t' \leq k + t + 1$, if there is a $(k',t')$-difficult mistake tree with respect to a hypothesis class $\calH$, then $\calS(\calH,t') \geq \binom{t'}{\leq k'+1}$.
Now suppose $\ELdim(\calH, k) \geq t$. By Lemma~\ref{lem:eldimrecursive}, there exists $(x,y)$ such that $\ELdim(\calH[(x,y)], k) \geq t - 1$ and $\ELdim(\calH[(x,-y)], k-1) \geq t - 1$.
Thus by inductive hypothesis, $\calS(\calH[(x,y)],t-1) \geq \binom{t-1}{\leq k+1}$ and $\calS(\calH[(x,-y)],t-1) \geq \binom{t-1}{\leq k}$. Therefore,
\[ \calS(\calH,t) \geq \calS(\calH[(x,y)],t-1) + \calS(\calH[(x,-y)],t-1) \geq \binom{t}{\leq k+1}\]
This completes the induction.
\end{proof}

\begin{proof}[Proof of Theorem~\ref{thm:thresholds}]
Note that $\calS(\calH, t) \leq |\calH| = n$, therefore by Lemma~\ref{lem:growth},
$\ELdim(\calH, k) \leq \max \cbr{t: \binom{t}{\leq k+1} \leq n}$.

On the other hand, Lemma~\ref{lem:optmt} implies that for all $m$ such that $\binom{m}{\leq k+1} \leq n$, there is a $(m,k)$-difficult extended mistake tree with respect to $\calH$. Hence
$\ELdim(\calH, k) \geq \max \cbr{t: \binom{t}{\leq k+1} \leq n}$
Combining the lower and upper bound, we get the theorem.
\end{proof}
\begin{lemma}
Consider the set of threshold classifiers $\calH = \cbr{2I(x \leq t) - 1: t \in \cbr{t_1, \ldots, t_n}}$. If integers $k \geq 0$ and $m \geq 0$ are such that $\binom{m}{k+1} \leq n$, then $\calH$ has a $(k,m)$-difficult mistake tree.
\label{lem:optmt}
\end{lemma}

%This gives a constructive proof of Theorem~\ref{thm:equivalence} for the special case when $\calH$ is the set of threshold classifiers.
\begin{proof}
We prove the lemma by joint induction on $(k,m)$.

\paragraph{Base Case:} Consider $k = 0$ or $m = 0$.
\begin{enumerate}[(1)]
\item For $k = 0$, $\binom{m}{\leq k} = m+1$. We show a construction of $T_{0,m}$, a $(0,m)$-difficult extended mistake tree in Figure~\ref{fig:basecasethreshold}. It can be seen that the resulting tree $T_{0,m}$ is $(0,m)$-difficult, as the only root to leaf path using no solid edges corresponds to examples $(t_2,+1)$, $\ldots$, $(t_{m+1}, +1)$, which has length $m$.
\begin{figure}
\centering
% Generated with LaTeXDraw 2.0.8
% Thu Feb 11 19:07:46 PST 2016
% \usepackage[usenames,dvipsnames]{pstricks}
% \usepackage{epsfig}
% \usepackage{pst-grad} % For gradients
% \usepackage{pst-plot} % For axes
\scalebox{1} % Change this value to rescale the drawing.
{
\begin{pspicture}(0,-2.6692188)(15.449062,2.6692188)
\psline[linewidth=0.04cm,tbarsize=0.07055555cm 5.0]{|-|}(7.41,0.93578124)(14.95,0.9557812)
\psline[linewidth=0.04cm](8.33,1.0557812)(8.33,0.7757813)
\psline[linewidth=0.04cm](9.25,1.0357813)(9.25,0.75578123)
\psline[linewidth=0.04cm](14.23,1.0957812)(14.23,0.81578124)
\usefont{T1}{ppl}{m}{n}
\rput(8.334531,0.5257813){$t_1$}
\usefont{T1}{ppl}{m}{n}
\rput(9.294531,0.5257813){$t_2$}
\usefont{T1}{ppl}{m}{n}
\rput(14.374531,0.5257813){$t_{m+1}$}
\usefont{T1}{ppl}{m}{n}
\rput(1.4945313,2.4657812){$t_2$}
\psline[linewidth=0.04cm](1.27,2.2157812)(0.65,1.5757812)
\psline[linewidth=0.04cm](1.63,2.2557812)(2.37,1.5957812)
\usefont{T1}{ppl}{m}{n}
\rput(0.5445312,1.9857812){$-1$}
\usefont{T1}{ppl}{m}{n}
\rput(2.5645313,1.9857812){$+1$}
\usefont{T1}{ppl}{m}{n}
\rput(2.4345312,1.3457812){$t_3$}
\psline[linewidth=0.04cm](2.27,1.0957812)(1.65,0.45578125)
\psline[linewidth=0.04cm](2.63,1.1357813)(3.37,0.47578126)
\usefont{T1}{ppl}{m}{n}
\rput(1.4445312,0.87578124){$-1$}
\usefont{T1}{ppl}{m}{n}
\rput(3.5645313,0.87578124){$+1$}
\usefont{T1}{ppl}{m}{n}
\rput(5.2145314,-1.2342187){$t_{m+1}$}
\psline[linewidth=0.04cm](4.57,-1.5242188)(3.95,-2.1642187)
\psline[linewidth=0.04cm](4.93,-1.4842187)(5.67,-2.1442187)
\usefont{T1}{ppl}{m}{n}
\rput(3.7445312,-1.7542187){$-1$}
\usefont{T1}{ppl}{m}{n}
\rput(5.864531,-1.7542187){$+1$}
\usefont{T1}{ppl}{m}{n}
\rput(0.47453126,1.3057812){$h_1$}
\usefont{T1}{ppl}{m}{n}
\rput(1.6545312,0.20578125){$h_2$}
\usefont{T1}{ppl}{m}{n}
\rput(6.2945313,-2.4342186){$h_{m+1}$}
\usefont{T1}{ppl}{m}{n}
\rput(4.094531,-2.4142187){$h_{m}$}
\psline[linewidth=0.04cm,linestyle=dashed,dash=0.16cm 0.16cm](1.47,2.1557813)(2.23,1.5357813)
\psline[linewidth=0.04cm,linestyle=dashed,dash=0.16cm 0.16cm](2.51,1.0557812)(3.31,0.37578124)
\psline[linewidth=0.04cm,linestyle=dashed,dash=0.16cm 0.16cm](4.79,-1.5442188)(5.59,-2.2242188)
\usefont{T1}{ppl}{m}{n}
\rput(4.1045313,-0.47421876){$\ldots$}
\end{pspicture} 
}
\caption{Construction of $T_{0,m}$, an extended mistake tree given parameters $k = 0$ and $m \geq 0$. For each $i$, $h_i$ is defined as $h_i(x) := 2I(x \leq t_i)-1$.}
\label{fig:basecasethreshold}
\end{figure}
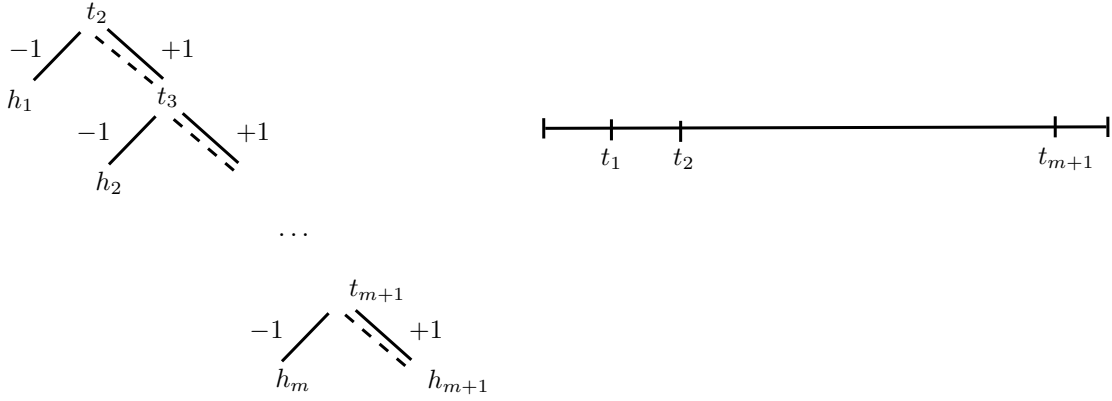

\item For $m = 0$ and integer $k$, $\binom{m}{\leq k} = 1$. The zeroth order extended mistake tree containing $h_{t_1}$ is a $(k,0)$-difficult extended mistake tree.
\end{enumerate}
\paragraph{Inductive Case:} For $k \geq 1$ and $m \geq 1$, assume the inductive hypothesis holds for $(k',m')$ such that $k' \leq k$, $m' \leq m$ and $k' + m' \leq k + m - 1$.
%, i.e. the $(k',m')$-difficult extended mistake trees $T_{k',m'}$ can be constructed for class of threshold functions $\calH = \cbr{I(x \leq t): t \in \cbr{t_1, \ldots, t_m}}$, when $n \geq \binom{m}{k+1}$,

We now construct $T_{k,m}$, a $(k,m)$-difficult extended mistake tree, using hypotheses in $\calH$.
Let $r_- = \binom{m-1}{\leq k}$, $r_+ = \binom{m-1}{\leq k+1}$. Consider hypothesis class $\calH_- = \cbr{2I(x \leq t) - 1: t \in \cbr{t_1, \ldots, t_{r_-}}}$ and $\calH_+ = \cbr{2I(x \leq t) - 1: t \in \cbr{t_{r_-+1}, \ldots, t_{r_- + r_+}}}$. Note that $r_- + r_+ \leq \binom{m-1}{\leq k} + \binom{m-1}{\leq k-1} \leq \binom{m}{\leq k} \leq n$, thus $\calH_-$ and $\calH_+$ are well defined.

Since $|\calH_-| \geq \binom{m-1}{\leq k}$, by inductive hypothesis, there is a $(k-1,m-1)$ difficult extended mistake tree $T_{k-1,m-1}$ with respect to $\calH_-$. Similarly, since $|\calH_+| \geq \binom{m-1}{\leq k+1}$, by inductive hypothesis, there is a $(k,m-1)$ difficult extended mistake tree $T_{k,m-1}$ with respect to $\calH_+$.

Now Let $x$ be a real number in $(t_{r_-},t_{r_- + 1})$, it can be seen that all hypotheses in $\calH_-$ classifies $x$ as $-1$ and all hypothesis in $\calH_+$ classifies $x$ as $+1$.
We construct $T_{k,m}$ as in Figure~\ref{fig:inductivecasethreshold}, where $x$ is at the root, and its downward left solid edge connect to $T_{k-1,m-1}$; its downward right solid edge and downward dashed edge connects to $T_{k,m-1}$. Note that $T_{k,m}$ is a valid extended mistake tree, since all hypotheses at the leaves in $T_{k-1,m-1}$ (resp. $T_{k-1,m}$) classifies $x$ as $-1$ (resp. $+1$).
By Lemma~\ref{lem:recursive}, $T_{k,m}$ is $(k,m)$-difficult.
%It can be seen that the resulting tree $T$ is $(k,d)$-difficult: suppose any path following the left substree must use solid edges $(x,-1)$, then since the left subtree is $(k-1,d-1)$-difficult, any path following the left subtree using at most $k$ solid edges must be of length at least $d$. Similarly, any path following the right subtree using at most $k-1$ solid edges must be of length at least $d$.
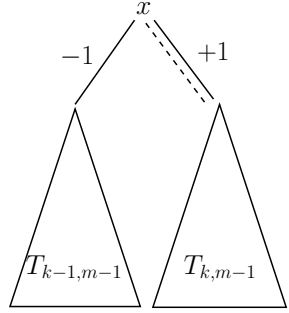
\begin{figure}
\centering
% Generated with LaTeXDraw 2.0.8
% Mon Jan 11 22:04:36 PST 2016
% \usepackage[usenames,dvipsnames]{pstricks}
% \usepackage{epsfig}
% \usepackage{pst-grad} % For gradients
% \usepackage{pst-plot} % For axes
\scalebox{0.5} % Change this value to rescale the drawing.
{
\huge
\begin{pspicture}(0,-4.076719)(7.3960705,4.116719)
\usefont{T1}{ppl}{m}{n}
\rput(3.5845313,3.9132812){$x$}
\usefont{T1}{ppl}{m}{n}
\rput(1.8145312,2.6332812){$-1$}
\psline[linewidth=0.04cm](3.34,3.5832813)(1.76,1.3432813)
\psline[linewidth=0.04cm](3.86,3.5832813)(5.44,1.4832813)
\usefont{T1}{ppl}{m}{n}
\rput(5.434531,2.7332811){$+1$}
\pstriangle[linewidth=0.04,dimen=outer](1.76,-4.056719)(3.52,5.34)
\pstriangle[linewidth=0.04,dimen=outer](5.6,-4.076719)(3.6,5.48)
\psline[linewidth=0.04cm,linestyle=dashed,dash=0.16cm 0.16cm](3.64,3.5032814)(5.2,1.4032812)
\usefont{T1}{ppl}{m}{n}
\rput(1.7345313,-3.0267189){$T_{k-1,m-1}$}
\usefont{T1}{ppl}{m}{n}
\rput(5.614531,-3.0267189){$T_{k,m-1}$}
\end{pspicture} 
\normalfont
}
\caption{Construction of $T_{k,m}$, an extended mistake tree given parameters $k \geq 1$ and $m \geq 1$, from $T_{k,m-1}$ and $T_{k-1,m-1}$.}
\label{fig:inductivecasethreshold}
\end{figure}
\end{proof}

\begin{lemma}
Let $\calC^l$ be the class of unions of at most $l$ singletons. Then
\[ \calS(\calC^l, t) = \binom{t}{\leq l}\]
\label{lem:tscus}
\end{lemma}
\begin{proof}
We show the equality by showing the inequality in both directions.
\begin{enumerate}[(1)]
\item $S(\calC^l, t) \leq \binom{t}{\leq l}$ From Lemma~\ref{lem:ldimgrowth}.

\item Consider a $\calX$-valued tree $\x$ with all its elements distinct. Then, consider the set
\[ S(\calC^l, \x) = \cbr{(\epsilon_1, \ldots, \epsilon_t): \exists h \in \calC^l, h(\x_s(\epsilon)) = \epsilon_s, s = 1,2,\ldots,t }\]
We claim that $S(\calC^l, \x)$ contains $\cbr{\epsilon = (\epsilon_1, \ldots, \epsilon_t): |\cbr{s: \epsilon_s = -1}| \leq l}$.
Indeed, for any element in $\cbr{\epsilon = (\epsilon_1, \ldots, \epsilon_t): |\cbr{s: \epsilon_s = -1}| \leq l}$, the hypothesis $h = 1 - 2I(x \in \cbr{\x_s(\epsilon): \epsilon_s = -1}) \in \calC^l$ satisfies that $h(\x_s(\epsilon)) = \epsilon_s$, for $s = 1,2,\ldots,t$. Hence $S(\calC^l, \x) \geq \binom{t}{\leq l}$, implying $\calS(\calC^l, t) \geq \binom{t}{\leq l}$.
\end{enumerate}

In summary, $S(\calC^l, t) = \binom{t}{\leq l}$.
\end{proof}

\begin{proof}[Proof of Theorem~\ref{thm:usingletons}]
We show the equality by showing the inequality in both directions.
\begin{enumerate}[(1)]
\item Consider the case that $k \leq l - 1$. By Lemma~\ref{lem:infmt},
for any integer $m$, there is a $(k,m)$-difficult extended mistake
tree with respect to $\calC^l$. Thus, $\ELdim(\calC^l, k) = \infty$.

\item Consider the case that $k \geq l$.
By Lemma~\ref{lem:tscus}, $\calS(\calC^l,t) = \binom{t}{\leq l}$.
By Theorem~\ref{thm:growth},
\[ \ELdim(\calC^l, k) \leq \max\left\{t: \binom{t}{\leq k+1} \leq \binom{t}{\leq l}\right\} = l \]
This gives that $\ELdim(\calC^l, k) \leq l$.

On the other hand, $\calC^l$ has a mistake tree $T$ of depth $l$.
Consider the following modification of $T$: for each internal node,
add a dashed downward edge to its right child. It can be seen that the resulting tree,
$\tilde{T}$, is an $(l,l)$-difficult extended mistake tree.
Therefore $T'$ is also a $(k,l)$-difficult mistake tree, which gives that $\ELdim(\calC^l, k) \geq l$.
\end{enumerate}
Hence, we conclude that $\ELdim(\calC^l, k) = l$.
\end{proof}
%\cz{Key recurrence: $F(N,d,k) = 1 + \min(F(N-1, k-1, k-1), F(N-1, d, k))$; $F(N,0,k) = 0$; $F(N,d,0) = N$ if $d \geq 1$.}

Recall that $\calC^l$ is the class of union of at most $l$ singletons in instance domain $\calX$. That is, hypotheses that take value $+1$ on $\calX$, except for at most $l$ points.
\begin{lemma}
Suppose we are given an infinite domain $\calX$ and an integer $l \geq 1$.
Then for any integer $m \geq 0$, there exists a $(l-1,m)$-difficult extended
mistake tree with respect to hypothesis class $\calC^l$, such that all its dashed edges are labeled $+1$.
\label{lem:infmt}
\end{lemma}
\begin{proof}
By induction on $l$.

\paragraph{Base Case:} For $l = 1$, the construction of the required extended mistake
tree with respect to $\calC^1$ is given in Figure~\ref{fig:base}.
Note that the tree is $(0,m)$-difficult, and all its dashed edges are labeled $+1$.

\paragraph{Inductive Case:} Suppose the inductive hypothesis holds for any $l' \leq l-1$.
Now pick an arbitrary $x \in \calX$. Fix integer $m$.
Consider $(\calX_1, \calX_2)$, a partition of $\calX \setminus \cbr{x}$,
where both $|\calX_1|$ and $|\calX_2|$ are infinite.

By inductive hypothesis, there is a $(l-1, m)$-difficult extended mistake tree $T_+$
with respect to $\calC^l$ on domain $\calX_1$, such that all its dashed edges are labeled $+1$.
Since for any $h \in \calC^{l-1}$, there exists $h' \in \calC^l[(x,+1)]$
such that $h \equiv h'$ on $\calX_1$, we can modify $T_+$'s leaves such that
they all correspond to hypotheses in $\calC^l[(x,+1)]$, getting a new extended mistake tree $\tilde{T}_+$.

Similarly, by inductive hypothesis, there is a $(l-2, m)$-difficult extended mistake tree $T_-$ with respect to $\calC^{l-1}$ on domain $\calX_2$, such that all its dashed edges are labeled $+1$. Since for any $h \in \calC^l$, there exists $h' \in \calC^l[(x,-1)]$ such that $h \equiv h'$ on $\calX_2$, we can modify $T_-$'s leaves such that they all correspond to hypotheses in $\calC^l[(x,-1)]$, getting a new extended mistake tree $\tilde{T}_-$.

Now consider the extended mistake tree $T$ rooted at $x$, with its left subtree as $\tilde{T}_-$ and right subtree as $\tilde{T}_+$. The dashed downward edge of root is linked to its right child, i.e. has label $+1$. Note that $T$ is a valid extended mistake tree, since all hypotheses at the leaves in $\tilde{T}_-$ (resp. $\tilde{T}_+$) classifies $x$ as $-1$ (resp. $+1$).
By Lemma~\ref{lem:recursive}, $T$ is $(l-1,m+1)$-difficult, hence $(l-1,m)$-difficult. Additionally, all its dashed edges are labeled $+1$. Since the choice of $m$ is arbitrary, this completes the induction.
\end{proof}

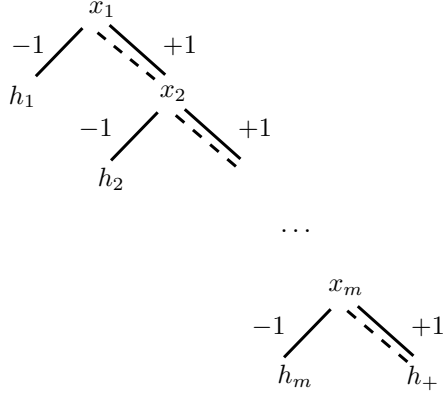
\begin{figure}
\centering
% Generated with LaTeXDraw 2.0.8
% Fri Feb 12 02:31:48 PST 2016
% \usepackage[usenames,dvipsnames]{pstricks}
% \usepackage{epsfig}
% \usepackage{pst-grad} % For gradients
% \usepackage{pst-plot} % For axes
\scalebox{1} % Change this value to rescale the drawing.
{
\begin{pspicture}(0,-2.6692188)(6.4090624,2.6692188)
\usefont{T1}{ppl}{m}{n}
\rput(1.5345312,2.4657812){$x_1$}
\psline[linewidth=0.04cm](1.27,2.2157812)(0.65,1.5757812)
\psline[linewidth=0.04cm](1.63,2.2557812)(2.37,1.5957812)
\usefont{T1}{ppl}{m}{n}
\rput(0.5445312,1.9857812){$-1$}
\usefont{T1}{ppl}{m}{n}
\rput(2.5645313,1.9857812){$+1$}
\usefont{T1}{ppl}{m}{n}
\rput(2.4745312,1.3457812){$x_2$}
\psline[linewidth=0.04cm](2.27,1.0957812)(1.65,0.45578125)
\psline[linewidth=0.04cm](2.63,1.1357813)(3.37,0.47578126)
\usefont{T1}{ppl}{m}{n}
\rput(1.4445312,0.87578124){$-1$}
\usefont{T1}{ppl}{m}{n}
\rput(3.5645313,0.87578124){$+1$}
\usefont{T1}{ppl}{m}{n}
\rput(4.7745314,-1.2342187){$x_m$}
\psline[linewidth=0.04cm](4.57,-1.5242188)(3.95,-2.1642187)
\psline[linewidth=0.04cm](4.93,-1.4842187)(5.67,-2.1442187)
\usefont{T1}{ppl}{m}{n}
\rput(3.7445312,-1.7542187){$-1$}
\usefont{T1}{ppl}{m}{n}
\rput(5.864531,-1.7542187){$+1$}
\usefont{T1}{ppl}{m}{n}
\rput(0.47453126,1.3057812){$h_1$}
\usefont{T1}{ppl}{m}{n}
\rput(1.6545312,0.20578125){$h_2$}
\usefont{T1}{ppl}{m}{n}
\rput(5.7945313,-2.4342186){$h_+$}
\usefont{T1}{ppl}{m}{n}
\rput(4.094531,-2.4142187){$h_{m}$}
\psline[linewidth=0.04cm,linestyle=dashed,dash=0.16cm 0.16cm](1.47,2.1557813)(2.23,1.5357813)
\psline[linewidth=0.04cm,linestyle=dashed,dash=0.16cm 0.16cm](2.51,1.0557812)(3.31,0.37578124)
\psline[linewidth=0.04cm,linestyle=dashed,dash=0.16cm 0.16cm](4.79,-1.5442188)(5.59,-2.2242188)
\usefont{T1}{ppl}{m}{n}
\rput(4.1045313,-0.47421876){$\ldots$}
\end{pspicture} 
}
\caption{A $(0,m)$-difficult extended mistake tree with respect to $\calC^1$. $x_1, \ldots, x_m$ are distinct elements in $\calX$, and for each $i$, $h_i$ is defined as $h_i(x) := 1 - 2 I(x = x_i)$. $h_+$ is the constant function $+1$.}
\label{fig:base}
\end{figure}

\section{Proofs from Section~\ref{sec:nonrealizable}}

\begin{proof}[Proof of Theorem~\ref{thm:finitelb}]
Let $\calX$ be an infinite set. Let $\calH$ be the hypothesis class containing only one hypothesis $h \equiv +1$. Note that $\calH^l = \calC^l$ and by Theorem~\ref{thm:usingletons}, $\ELdim(\calH^l, k) = \infty$ for $k < l$. By Lemma~\ref{lem:emtusage}, the theorem follows.
\end{proof}

\begin{proof}[Proof of Theorem~\ref{thm:infinitelb}]
Let $\calX$ be an infinite set. Let $\calH$ be the $\calC^d$, the class of unions of at most $d$ singletons. Note that $\calH^l = \calC^{l+d}$ and by Theorem~\ref{thm:usingletons}, $\ELdim(\calC^{l+d}, k) = \infty$ for $k < l+d$. By Lemma~\ref{lem:emtusage}, the theorem follows.
\end{proof}

We will need the following result regarding the tree shattering coefficient of the ``product'' of two hypothesis classes.

\begin{lemma}
Suppose $\calH_1, \calH_2$ are two hypothesis classes. If $\calH = \calH_1 \cdot \calH_2$, then for all integers $t \geq 0$,
\[ \calS(\calH,t) \leq \calS(\calH_1,t) \cdot \calS(\calH_2,t) \]
\label{lem:growthcompose}
\end{lemma}

We first show a basic property of tree shattering coefficient.
\begin{lemma}
For hypothesis classes $\calH_1$, $\calH_2$, $\calS(\calH_1 \cup \calH_2, t) \leq \calS(\calH_1,t) + \calS(\calH_2,t)$.
\label{lem:growthunion}
\end{lemma}
\begin{proof}
For any depth-$t$ tree $\x$, we have that
\[ S(\calH_1 \cup \calH_2, \x) \subseteq S(\calH_1, \x) \cup S(\calH_2, \x) \]
Therefore,
\[ |S(\calH_1 \cup \calH_2, \x)| \leq |S(\calH_1, \x)| + |S(\calH_2, \x)| \leq \calS(\calH_1,t) + \calS(\calH_2,t) \]
Since the choice of $\x$ is arbitrary, we get
\[ \calS(\calH_1 \cup \calH_2, t) \leq \calS(\calH_1,t) + \calS(\calH_2,t). \qedhere \]
\end{proof}

\begin{proof}[Proof of Lemma~\ref{lem:growthcompose}]
By induction on $t$.
\paragraph{Base Case:} Consider $t = 0$. If one of $\calS(\calH_1,0)$, $\calS(\calH_2,0)$ is $0$, this implies $\calH_1 = \emptyset$ or $\calH_2 = \emptyset$. Therefore, $\calH = \emptyset$, the result holds.
Otherwise, both $\calS(\calH_1,0)$ and $\calS(\calH_2,0)$ are at least $1$. In this case $\calH$ is nonempty, thus $1 = \calS(\calH,0) \leq 2^0 = 1$, the result also hold.

\paragraph{Inductive Case:} Given $t \geq 1$, assume the inductive hypothesis $\calS(\calF_1 \cdot \calF_2,t-1) \leq \calS(\calF_1,t-1) \cdot \calS(\calF_2,t-1)$ holds for any hypothesis classes $\calF_1$, $\calF_2$.
Fix $x \in \calX$. Note that $\calH[(x,+1)] = (\calH_1[(x,+1)] \cdot \calH_2[(x,+1)]) \cup (\calH_1[(x,-1)] \cdot \calH_2[(x,-1)])$. Therefore,
\begin{eqnarray*}
&&\calS(\calH[(x,+1)],t-1) \\
&\leq& \calS(\calH_1[(x,+1)] \cdot \calH_2[(x,+1)],t-1) + \calS(\calH_1[(x,-1)] \cdot \calH_2[(x,-1)],t-1) \\
&\leq& \calS(\calH_1[(x,+1)],t-1) \calS(\calH_2[(x,+1)],t-1) + \calS(\calH_1[(x,-1)],t-1) \calS(\calH_2[(x,-1)],t-1)
\end{eqnarray*}
where the first inequality is Lemma~\ref{lem:growthunion}, the second inequality is by inductive hypothesis.

Likewise, we have
\begin{eqnarray*}
\calS(\calH[(x,-1)],t-1) \leq \calS(\calH_1[(x,-1)],t-1) \calS(\calH_2[(x,+1)],t-1) + \calS(\calH_1[(x,+1)],t-1) \calS(\calH_2[(x,-1)],t-1)
\end{eqnarray*}

Therefore,
\begin{eqnarray*}
&&\calS(\calH[(x,-1)],t-1) + \calS(\calH[(x,+1)],t-1) \\
&\leq& (\calS(\calH_1[(x,-1)],t-1) + \calS(\calH_1[(x,+1)],t-1))(\calS(\calH_2[(x,-1)],t-1) + \calS(\calH_2[(x,+1)],t-1)) \\
&\leq& \calS(\calH_1,t) \calS(\calH_2,t)
\end{eqnarray*}
where the second inequality is from Lemma~\ref{lem:eldimrecursive}.
Since the choice of $x$ is arbitrary, we get
\[ \calS(\calH,t) = \max_{x \in \calX} (\calS(\calH[(x,-1)],t-1) + \calS(\calH[(x,+1)],t-1)) \leq \calS(\calH_1,t) \calS(\calH_2,t). \qedhere \]
\end{proof}

\begin{proof}[Proof of Lemma~\ref{lem:finiteub}]
Note that by Lemmas~\ref{lem:ldimgrowth} and~\ref{lem:growthcompose},
\[ \calS(\calH^l,t) \leq \calS(\calC^l,t) \cdot \calS(\calH,t) = \binom{t}{\leq l} \cdot \calS(\calH,t) \leq |\calH| \binom{t}{\leq l}\]
Hence,
\[ \ELdim(\calH^l, k) \leq \max\left\{t: \binom{t}{\leq k+1} \leq |\calH| \binom{t}{\leq l}\right\}\]
Now, consider any $t$ such that $t \geq 2l$ and
\[ \binom{t}{\leq k+1} \leq |\calH| \binom{t}{\leq l} \]
Since $\binom{t}{\leq k+1} \geq \binom{t}{k+1} \geq (\frac{t}{k+1})^{k+1}$, and $\binom{t}{\leq l} \leq (\frac{e t}{l})^l$ for $t \geq 2l$, we get
\[ (\frac{t}{k+1})^{k+1} \leq |\calH| (\frac{e t}{l})^l \]
Hence,
\[ t^{k+1-l} \leq |\calH| \frac{(k+1)^{k+1}}{l^l} \]
Since $\frac{(k+1)^{k+1}}{l^l} \leq (e(k+1))^{k+1-l}$, we get
\[ t^{k+1-l} \leq |\calH| (e(k+1))^{k+1-l} \]
That is, $t \leq e(k+1) |\calH|^{\frac{1}{k+1-l}}$.

In summary,
\[ \ELdim(\calH^l, k) \leq \max(2l, e(k+1) |\calH|^{\frac{1}{k+1-l}}) = e(k+1) |\calH|^{\frac{1}{k+1-l}} \]
where the equality uses the fact that $k \geq l$.
\end{proof}

\begin{proof}[Proof of Lemma~\ref{lem:infiniteub}]
Note that by Lemmas~\ref{lem:ldimgrowth} and~\ref{lem:growthcompose},
\[ \calS(\calH^l,t) \leq \calS(\calC^l,t) \cdot \calS(\calH,t) = \binom{t}{\leq l} \cdot \calS(\calH,t) \leq \binom{t}{\leq d} \binom{t}{\leq l}\]
Hence,
\[ \ELdim(\calH^l, k) \leq \max\left\{t: \binom{t}{\leq k+1} \leq \binom{t}{\leq d} \binom{t}{\leq l}\right\}\]
Now, consider any $t$ such that $t \geq 2l$, $t \geq 2d$ and
\[ \binom{t}{\leq k+1} \leq \binom{t}{\leq d} \binom{t}{\leq l} \]
Since $\binom{t}{\leq k+1} \geq \binom{t}{k+1} \geq (\frac{t}{k+1})^{k+1}$, $\binom{t}{\leq d} \leq (\frac{e t}{d})^d$ for $t \geq 2d$, and $\binom{t}{\leq l} \leq (\frac{e t}{l})^l$ for $t \geq 2l$, we get
\[ (\frac{t}{k+1})^{k+1} \leq  (\frac{e t}{d})^d (\frac{e t}{l})^l \]
Hence,
\[ t^{k+1-l-d} \leq e^{l+d} \frac{(k+1)^{k+1}}{l^l d^d} \]
Since
\begin{eqnarray*}
&& \frac{(k+1)^{k+1}}{l^l d^l} \\
&=& (1 + \frac{k+1-l}{l})^l (1 + \frac{k+1-d}{d})^d (k+1)^{k+1-l-d} \\
&\leq& e^{2k+2-l-d} (k+1)^{k+1-l-d}
\end{eqnarray*}
we get
\[ t \leq (k+1) \cdot e^{\frac{2k+2}{k+1-l-d}} \]
In summary,
\[ \ELdim(\calH^l, k) \leq \max(2l, 2d, (k+1) \cdot e^{\frac{2k+2}{k+1-l-d}}) ) =  (k+1) \cdot e^{\frac{2k+2}{k+1-l-d}} \]
where the equality uses the fact that $k \geq l+d$.
\end{proof}

\begin{proof}[Proof of Theorem~\ref{thm:finitelbrand}]
Let $\calX$ be an infinite set. $x_1$, $x_2$, \ldots is a sequence of distinct elements from $\calX$. Let $\calH$ be the hypothesis class containing only one hypothesis $h \equiv +1$. Note that $\calH^l = \calC^l$.
Let $\epsilon = 1 - k/l > 0$, thus $k = l(1-\epsilon)$. Fix integer $m = \lceil \frac{2}{\epsilon}(a + l) + 2l\rceil$. By Lemma~\ref{lem:infmt}, $\calC^l$ has a $(l-1,m)$-difficult extended mistake tree $T$.

We define the following strategy by the adversary based on $T$. At time $t = 1$, the adversary chooses the example $x_1$ corresponding to the root of $T$, and shows it to the learner. If $p_{1,+} > 1-\epsilon$, then it reveals label $y_1 = -1$ and follows the downward solid edge labeled $-1$ to reach the left child of the root; otherwise it reveals label $y_1 = +1$ and follows the downward dashed edge labeled $+1$ to reach the right child of the root. At time $t \geq 2$, suppose the adversary reaches a node with example $x_t$, then $x_t$ is shown to the learner, and one of the downward edges adjacent to this node is followed by the same rule. The interaction comes to an end when a leaf is reached. It can be seen that the realizability assumption is maintained.

Consider an Algorithm $\calA$ that guarantees a cumulative mistake penalty at most $k$.

\begin{enumerate}[(1)]
%over the first $m$ rounds, the adversary always presents some example $x_t$. In other words,
\item We claim that the interaction between the learner and the adversary lasts for at least $m$ rounds. To see this, note that $\calA$ predicts at most $l-1$ times such that $p_{t,+} > 1-\epsilon$. Assume this is not the case, that is,
\[ | \cbr{t \in [m]: p_{t,+} > 1-\epsilon } | \geq l \]
Suppose the first $l$ times $\calA$ predicts $p_{t,+} > 1-\epsilon$ are $1 \leq t_1 < \ldots < t_l \leq m$. Then, according to the adversary's strategy, $y_{t_1} = \ldots = y_{t_l} = -1$. Thus, the cumulative mistake penalty made by $\calA$ up to time $t_l$ is at least
\[ \sum_{i=1}^l p_{t_i,+} > l(1-\epsilon) = k \]
This implies that $\calA$ has a cumulative mistake penalty $>k$, contradiction. Therefore throughout the interaction, the number of solid edges used is at most $l-1$. Since $T$ is $(l-1,m)$-difficult, any path that going downward from the root using $l-1$ solid edges must be of length at least $m$, hence the interaction between the learner and the adversary lasts for at least $m$ rounds.

\item We claim that over the first $m$ rounds, there are at most $\frac{2l}{\epsilon}$ rounds such that $\calA$ predicts $p_{t,-} > \epsilon/2$ and $p_{t,+} \leq 1-\epsilon$. Assume this is not the case, that is,
\[ | \cbr{t \in [m]: p_{t,-} > \epsilon/2 \wedge p_{t,+} \leq 1-\epsilon} | \geq \frac{2l}{\epsilon} \]
Suppose the first $g = \lceil \frac{2l}{\epsilon} \rceil$ times $\calA$ predicts $-1$ are $1 \leq s_1 < \ldots < s_g \leq m$. Then, according to the adversary's strategy, $y_{s_1} = \ldots = y_{s_g} = +1$. Thus the cumulative mistake penalty made by $\calA$ up to time $s_g$ is at least
\[ \sum_{i=1}^g p_{s_i,-} > g \cdot \frac{\epsilon}{2} \geq l(1-\epsilon) = k\]
This implies that $\calA$ has a cumulative mistake penalty $>k$ over time, contradiction.
\end{enumerate}

Therefore, among the first $m$ rounds, there are at most $l + (l + \frac{2l}{\epsilon}) = 2l + \frac{2l}{\epsilon}$ rounds such that $p_{t,+} > 1-\epsilon$ or $p_{t,-} > \epsilon/2$. Thus there are at least $(m - \frac{2l}{\epsilon} - 2l)$ rounds such that $p_{t,+} \leq 1-\epsilon$ and $p_{t,-} \leq \epsilon/2$, implying $1 - p_{t,+} - p_{t,-} \geq \epsilon/2$. Thus, the cumulative abstention penalty up to time $m$ is at least
\[ (m - \frac{2l}{\epsilon} - 2l) \cdot \frac{\epsilon}{2} \geq a. \qedhere \]
\end{proof}

\begin{proof}[Proof of Theorem~\ref{thm:infinitelbrand}]
Let $\calX$ be an infinite set. Let $\calH$ be the $\calC^d$, the class of unions of at most $d$ singletons. Note that $\calH^l = \calC^{l+d}$. Hence $l$-bias assumption with respect to $\calH$ is equivalent to $\calC^{l+d}$-realizability. The rest of the proof is analogous to the proof of Theorem~\ref{thm:finitelbrand}.
\end{proof}

\end{document}